\documentclass[10pt,journal,compsoc]{IEEEtran}

\usepackage{times}
\usepackage{graphicx}
\usepackage{amsmath}
\usepackage{amssymb}

\usepackage[american]{babel}
\usepackage{microtype}
\usepackage{multirow}
\usepackage{booktabs}
\graphicspath{{figure/}}
\usepackage[caption=false]{subfig}
\usepackage{siunitx}
\usepackage{xcolor}
\usepackage{colortbl}
\usepackage{diagbox}
\allowdisplaybreaks[3]

\definecolor{Shade}{rgb}{0.82, 0.9, 0.94}

\usepackage{bm}
\def\vzero{{\bm{0}}}
\def\vb{{\bm{b}}}
\def\vf{{\bm{f}}}
\def\vh{{\bm{h}}}

\def\vl{{\bm{l}}}
\def\vp{{\bm{p}}}
\def\vv{{\bm{v}}}
\def\vu{{\bm{u}}}
\def\vw{{\bm{w}}}
\def\vx{{\bm{x}}}
\def\vy{{\bm{y}}}
\def\mW{{\bm{W}}}
\def\mR{{\bm{R}}}
\def\sB{{\mathbb{B}}}
\def\sO{{\mathbb{O}}}
\def\sN{{\mathbb{N}}}
\def\sL{{\mathbb{L}}}
\newcommand{\E}{\mathbb{E}}
\newcommand{\R}{\mathbb{R}}
\newcommand{\N}{\mathbb{N}}
\newcommand{\sigmoid}{\sigma}
\newcommand{\Loss}{\mathcal{L}}
\DeclareMathOperator{\sign}{sign}
\DeclareMathOperator{\dir}{unit}
\DeclareMathOperator{\anglefn}{\angle}
\newcommand*{\sphere}[1]{{\mathbb{S}^{{#1} - 1}}}
\def\deri#1#2{\frac{\mathrm{d} {#1}}{\mathrm{d}{#2}}}

\newcommand*{\tran}{^{\mkern-1.5mu\mathsf{T}}}
\newcommand*{\normtwo}[1]{{\left\|{#1}\right\|_2}}
\newcommand*{\prob}[1]{{\Pr{\left\{{#1}\right\}}}}
\newcommand*{\expect}[2]{{\E_{{#2}}{\left[{#1}\right]}}}

\newcommand*{\pd}[1]{{p{\left\{{#1}\right\}}}}
\newcommand*{\pdfunc}[1]{{p{\left({#1}\right)}}}

\newtheorem{claim}{Claim}
\newtheorem{lemma}[claim]{Lemma}
\newtheorem{corollary}{Corollary}[claim]

\newenvironment{proof}{{\noindent\it Proof.}\quad}{\hfill $\square$\par}

\def\Lemref#1{Lemma~\ref{#1}}

\def\Corref#1{Corollary~\ref{#1}}
\def\eqref#1{equation~\ref{#1}}
\def\Eqref#1{Equation~\ref{#1}}

\ifCLASSOPTIONcompsoc
\usepackage[nocompress]{cite}
\else
\usepackage{cite}
\fi
\usepackage{url}

\begin{document}
	
\title{Distilling Knowledge by Mimicking Features}

\author{Guo-Hua~Wang,
	Yifan~Ge,
	and~Jianxin~Wu,~\IEEEmembership{Member,~IEEE}
	\IEEEcompsocitemizethanks{\IEEEcompsocthanksitem All authors are with the State Key Laboratory for Novel Software Technology, Nanjing University, Nanjing 210023, China. J. Wu is the corresponding author.\protect\\
		E-mail: \{wangguohua, geyf, wujx\}@lamda.nju.edu.cn.}
	\IEEEcompsocitemizethanks{\IEEEcompsocthanksitem This research was partly supported by the National Natural Science Foundation of China under Grant 61772256 and Grant 61921006.}
}

\markboth{To appear in IEEE Trans. PAMI}%
{Distilling Knowledge by Mimicking Features}

\IEEEtitleabstractindextext{%
	\begin{abstract}
		Knowledge distillation (KD) is a popular method to train efficient networks (``student'') with the help of high-capacity networks (``teacher''). Traditional methods use the teacher's soft logits as extra supervision to train the student network. In this paper, we argue that it is more advantageous to make the student mimic the teacher's features in the penultimate layer. Not only the student can directly learn more effective information from the teacher feature, feature mimicking can also be applied for teachers trained without a softmax layer. Experiments show that it can achieve higher accuracy than traditional KD. To further facilitate feature mimicking, we decompose a feature vector into the magnitude and the direction. We argue that the teacher should give more freedom to the student feature's magnitude, and let the student pay more attention on mimicking the feature direction. To meet this requirement, we propose a loss term based on locality-sensitive hashing (LSH). With the help of this new loss, our method indeed mimics feature directions more accurately, relaxes constraints on feature magnitudes, and achieves state-of-the-art distillation accuracy. We provide theoretical analyses of how LSH facilitates feature direction mimicking, and further extend feature mimicking to multi-label recognition and object detection.
	\end{abstract}
	
	}

\maketitle

\IEEEdisplaynontitleabstractindextext
\IEEEpeerreviewmaketitle
\IEEEraisesectionheading{\section{Introduction}\label{sec:introduction}}

\IEEEPARstart{R}{ecently}, deep learning has achieved remarkable success in many visual recognition tasks. To deploy deep networks in devices with limited  resources, more and more efficient networks have been proposed \cite{mobilenet,ShuffleNetV2}. Knowledge distillation (KD)~\cite{KD} is a popular method to train these efficient networks (named ``student'') with the help of high-capacity networks (named ``teacher'').

Initial study of KD~\cite{KD} used the softmax output of the teacher network as the extra supervisory information for training the student network. However, the output of a high-capacity network is not significantly different from groundtruth labels. And, due to the existence of the classifier layer, the softmax output contains less information compared with the representation in the penultimate layer. These issues hinder the performance of a student model. In addition, it is difficult for KD to distill teacher models trained by unsupervised or self-supervised learning~\cite{RotNet,simclr,MoCoV2,MoCo}.

Feature distillation has received more and more attention in recent years~\cite{fitnet,FT,CRD,SSKD}. However, previous works only focused on distilling features in the middle layers~\cite{fitnet} or transforming the features~\cite{FT}. Few have addressed the problem of making the student directly mimic the teacher's feature in the penultimate layer. Distilling features in the middle layers suffers from the different architectures between teacher and student, while transforming the features may lose some information in the teacher. We believe it is a better way to \emph{directly mimic the feature for knowledge distillation, in which we only mimic the feature in the penultimate layer}. Compared with KD, it does not need the student model to learn a classifier from the teacher. Feature mimicking can be applied to a teacher trained by unsupervised, metric or self-supervised learning, and can be easily used when the teacher and student have different architectures. Furthermore, if the student features are the same as the teacher's, the classification accuracy will surely be the same, too.

Some reasons may explain why feature mimicking has not yet been popular in the literature. First, previous work used the mean squared loss ($\ell_2$ loss) to distill features. In this paper, we decompose a feature vector into the magnitude and the direction. The $\ell_2$ loss focuses on both magnitude and direction. But due to the different capacities, the student cannot mimic the teacher in its entirety. In fact, only the direction affects the classification result while the magnitude mainly represents the confidence of prediction~\cite{feature_norm_kd}. We find that different networks often have different feature magnitudes (cf. Table~\ref{tab:arch}). That inspires us to give more freedom to the student feature's magnitude. One possible approach to tackle this problem is to distill the feature after $\ell_2$-normalization~\cite{FT}. However, it will lose all magnitude information about the teacher feature and make the optimization difficult~\cite{normface}. In this paper, we propose a loss term which focuses on the feature direction \emph{and gives more freedom to its magnitude}, which alleviates the shortcomings of the $\ell_2$ loss (cf. Figure~\ref{fig:lsh}). 

Second, when teacher and student features have different dimensionalities, difficulty arises. To solve this problem, we split the final fully connected (FC) layer of the student network into two FC layers without non-linear activation in-between. The dimensionality of the first FC layer matches that of the teacher feature. The two FCs can be merged into one after training. Hence, no extra parameter or computation is added in the student's architecture during inference. 

Third, even though the feature structure of the student is the same as that of the teacher, their feature space may misalign (cf. Figure~\ref{fig:feature}). If we have the freedom to rotate and rescale the student's feature space, it will align to the teacher's feature space better. Thanks to our two FC structure in the proposed feature mimicking method, we demonstrate that the first FC layer can transform the student's feature space and make feature mimicking easier, which is particularly important when the student network is initialized using a pretrained model (i.e., the student has formed a basic feature space to finetune rather than a random feature space).

Our contributions are as follows. 

\begin{itemize}
	\item We argue that directly mimicking features in the penultimate layer is advantageous for knowledge distillation. It produces better performance than distilling logits after log-softmax (as in~\cite{KD}). It can be applied when the teacher and student have different architectures, while distilling features in the middle layers cannot.
	\item We claim that the feature's direction contains more effective information than its magnitude, and we should allow more freedom to the student feature's magnitude. We propose a loss term based on Locality-Sensitive Hashing (LSH)~\cite{LSH} to meet this requirement, and theoretically show why LSH fits this purpose.
	\item We propose a training strategy for mimicking features in transfer learning. With a pretrained student, we first transform its feature space to align to the teacher's, then finetune the student on the target dataset with our loss function. Our method is flexible and handles multi-label recognition well, while existing KD methods are difficult to apply to multi-label problems.
\end{itemize}
Our feature mimicking framework achieves state-of-the-art results on both single-label and multi-label recognition, and object detection tasks. 

The rest of this paper is organized as follows. First, we review the related work in Section~\ref{sec:related_work}. Then, we introduce our method for feature mimicking in Section~\ref{sec:method}, and mathematically analyze the effectiveness of it in Section~\ref{sec:theory}. Experimental results are reported and analyzed in Section~\ref{sec:experiments}. Finally, Section~\ref{sec:conclusion} concludes this paper.

\section{Related work}
\label{sec:related_work}

\textbf{Knowledge distillation} was first introduced in \cite{KD}, which proposed to use the teacher's soft logits after log-softmax as extra supervision to train the student. FitNet~\cite{fitnet} is the first work to distill the intermediate feature maps between teacher and student. Inspired by this, a variety of other feature-based knowledge distillation methods have been proposed. AT~\cite{AT} transfers the teacher knowledge to student by the spatial attention maps. AB~\cite{AB} proposes a knowledge transfer method via distillation of activation boundaries formed by hidden neurons. FitNet, AT and AB focus on activation maps of the middle layers, and it is difficult to apply them on cross-architecture settings. SP~\cite{SP} considers pairwise similarities of different features instead of mimicking the teacher's representation space. FSP~\cite{FSP} computes the inner product between features from two layers and treats it as the extra information to teach student. FT~\cite{FT} introduces a paraphraser to compress the teacher feature and uses the translator located at the student network to extract the student factors, then teaches the student by making student factors mimic teacher's compressed features. These methods transform the teacher's feature into other forms, which will lose some information in teacher features. In contrast, feature mimicking in the penultimate layer can apply on arbitrary teacher/student combinations and carry all information from the teacher. 

Recently, CRD~\cite{CRD} and SSKD~\cite{SSKD} take advantage of contrastive learning and transfer the structural knowledge of the teacher network to the student. In this paper, we argue that we can also achieve state-of-the-art by only mimicking features without explicitly considering the structural knowledge.

\textbf{Object detection} is a fundamental task in computer vision. Several previous works study knowledge distillation on the object detection task. ROI-mimic~\cite{ROI_mimic} mimics the features after ROI pooling. Fine-grained~\cite{DistillOD_FG} uses the ground truth bounding box to generate the foreground mask and distill the foreground features on the feature map. PAD~\cite{PAD} introduces the adaptive sample weighting to improve these distillation methods. In this paper, we will show that mimicking features in the penultimate layer works better. 

\textbf{Locality-sensitive hashing (LSH)} was first introduced in \cite{hash,approximate_nn}. With the help of p-stable distributions, \cite{LSH} extended the algorithm to the $\ell_2$ norm. With the rise of deep learning, hashing methods were widely used in image retrieval \cite{hashnet,DCH,deep_sup_hash,supervised_hash}. Most of them focused on how to learn good hash functions to transform images into compact codes. Different from that, we utilized LSH to help the student network to learn from the teacher network. To the best of our knowledge, we are the first to propose the use of LSH in distilling knowledge.

\section{Feature mimicking for knowledge distillation}
\label{sec:method}

\begin{figure*}[t]
	\centering
	\includegraphics[width=0.65\linewidth]{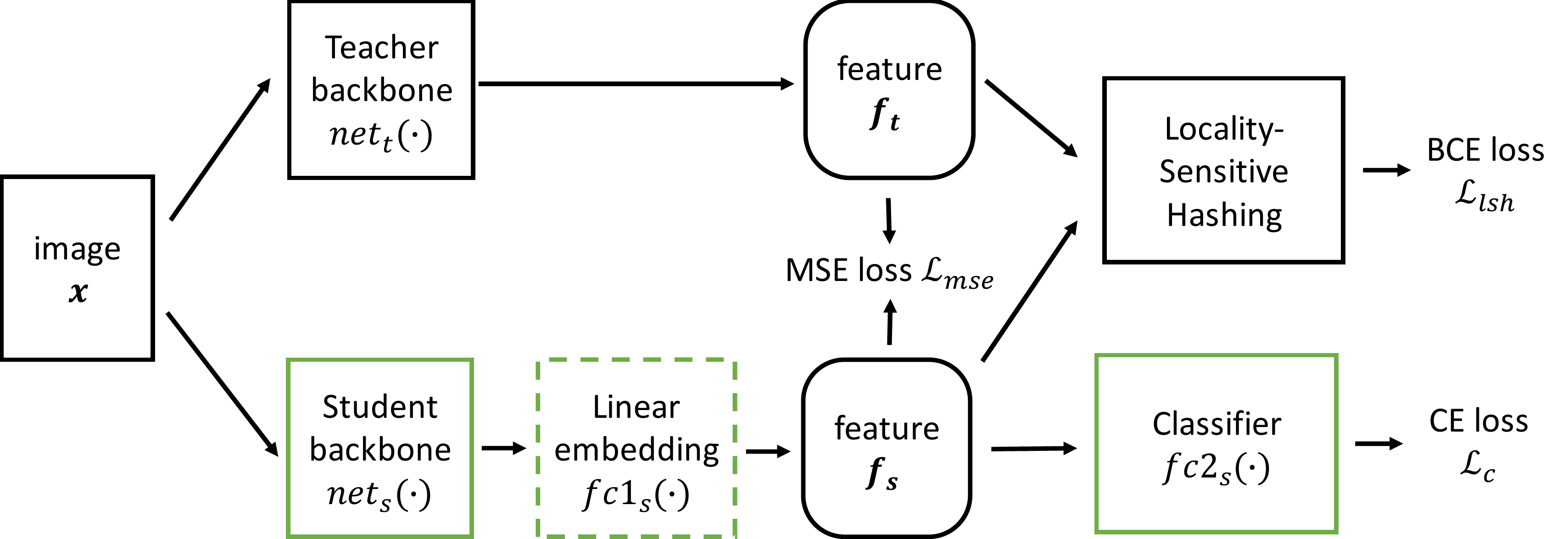}
	\caption{The pipeline of our method. We use a linear embedding layer to make sure the dimensionality of student's feature is the same as that of the teacher's. But, this embedding layer will be absorbed post-training. (This figure is best viewed in color.)}
	\label{fig:framework}
\end{figure*}

Figure~\ref{fig:framework} shows the pipeline of our method. Given an image $\vx$, the teacher backbone network extracts feature $\vf_t$, in which $\vf_t\in\R^{D_t}$ is the penultimate layer feature (after the global average pooling and before the final classifier or detection head). The student backbone network extracts feature $\vf_s$. To make the dimensionalities of $\vf_s$ and $\vf_t$ match, we add a linear embedding layer after the student backbone. Section~\ref{sec:The linear embedding layer} will introduce this module in detail. 

Three losses are used. $\Loss_c$ is the regular cross-entropy loss between the student output and the ground truth label of $\vx$. $\Loss_{mse}$ and $\Loss_{lsh}$ are used to make the student feature mimic the teacher's. More details about these two losses can be found in Section~\ref{sec:Locality-sensitive hashing}. More analyses are in Sections~\ref{sec:analysis} and \ref{sec:Ensemble all loss}. During training, modules with green boxes (student backbone, linear embedding and classifier) in Figure~\ref{fig:framework} need to be learned by back-propagation. Parameters in the teacher backbone and locality-sensitive hashing will \emph{not} change after initialization. Finally, Section~\ref{sec:Model initialization} discusses how to initialize our framework. We leave theoretical results for feature mimicking to Section~\ref{sec:theory}.

\subsection{The linear embedding layer}
\label{sec:The linear embedding layer}

When the dimensionality of the student's feature is different from that of the teacher's, we add a linear embedding layer before the student's classifier layer. Assume the dimensionality of student's features and teacher's are $D_s$ and $D_t$, respectively, the embedding layer is defined as
\begin{equation}
	fc1_s(\vf)=\mW_1\tran\vf+\vb_1\,,
\end{equation}
where $\mW_1\in\R^{D_s\times D_t}$ and $\vb_1\in\R^{D_t}$. The main advantage of this approach is that the embedding layer can be merged into the classifier without adding parameters or computation post-training. Assume the classifier is defined as
\begin{equation}
	fc2_s(\vf)=\mW_2\tran\vf+\vb_2\,,
\end{equation}
where $\mW_2\in\R^{D_t\times C}$ and $\vb_2\in\R^{C}$. Then, the final classifier for student can be computed by
\begin{align}
	fc_s(\vf)
	&=fc2_s(fc1_s(\vf)) \\
	&=(\mW_1\mW_2)\tran\vf+(\mW_2\tran\vb_1+\vb_2) \,.
\end{align}
$fc1_s$ and $fc2_s$ can be merged by setting the weights and bias for the final classifier as $\mW_1\mW_2$ and $\mW_2\tran\vb_1+\vb_2$, respectively.

This linear embedding layer shares similar idea as FSKD~\cite{FSKD}.  FSKD adds a $1\times 1$ conv at the end of each block of the student network and proves that the $1\times 1$ conv can be merged into the previous convolution layer. However, FSKD requires the teacher and student to share similar architectures, and adds more parameters during training.  Our method is more efficient and can be applied with different teacher/student architectures.

\begin{figure}[t]
	\centering
	\includegraphics[width=0.8\linewidth]{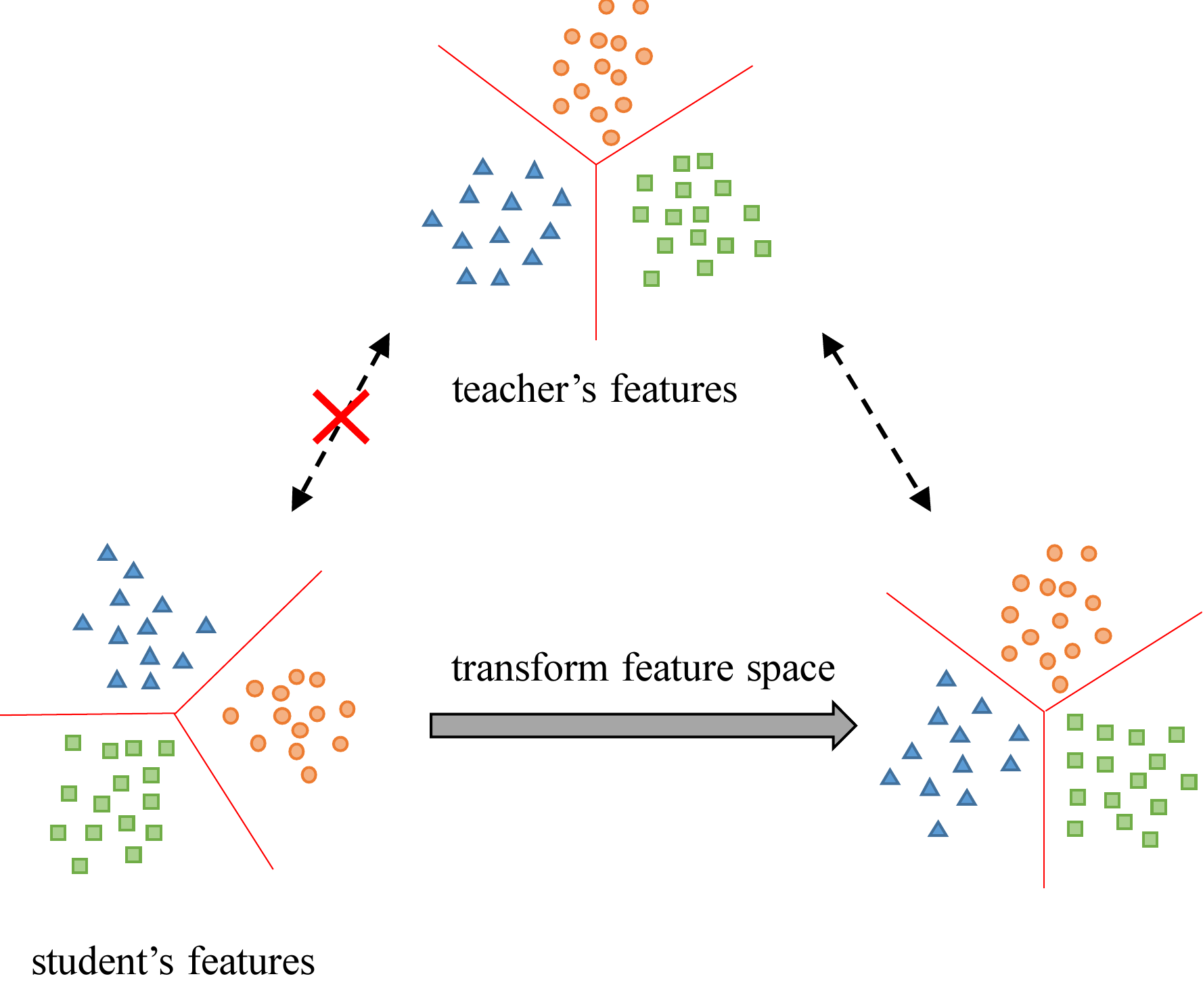}
	\caption{An illustration of the feature space misalignment issue. The points denote the features, and different colors with different shapes represent different classes. The student's feature space needs to rotate to align to the teacher's. (This figure is best viewed in color.)}
	\label{fig:feature}
\end{figure}

Even when the dimensionality of the student's feature is the same as that of the teacher's, the linear embedding layer may still be necessary. Because the teacher and the student may have significantly different network architectures, their feature spaces may be misaligned. The teacher and the student feature spaces, even when they encode the same semantic information, can still be subject to differences caused by transformations such as rotation and scaling. Figure~\ref{fig:feature} illustrates the feature space misalignment issue. Assume the penultimate layer feature is denoted by $\vf$ and the classifier's parameters are $\mW$ and $\vb$, respectively. The prediction can be computed by
\begin{align}
	\vp
	&=\mW\tran\vf+\vb \,.
\end{align}
Given any orthogonal matrix $\mR$, we have
\begin{align}
	\vp
	&=\mW\tran\mR\tran\mR\vf+\vb \\
	&=\mW_*\tran \vf_*+\vb \,,
\end{align}
where $\mW_*$ and $\vf_*$ are $\mR\mW$ and $\mR \vf$, respectively. That is, the feature space can be rotated without changing the prediction. 

Our linear embedding layer $fc1_s$ can learn any linear transformation (such as the above rotation $\mR$) to align the student's feature space to that of the teacher's. If we let the student mimic the teacher's features directly without aligning their feature spaces, the performance will be lower, especially when the student has been pretrained. Experimental validation of the importance of feature space alignment can be found in Section~\ref{sec:multi-label classification}.

\subsection{The LSH module}
\label{sec:Locality-sensitive hashing}

To mimic the teacher's feature, $\Loss_{mse}$ and $\Loss_{lsh}$ are used in our framework. $\Loss_{mse}$ is defined as
\begin{equation}
	\Loss_{mse} = \frac{1}{nD}\sum_{i=1}^{n}\|\vf_t(\vx_i) - \vf_s(\vx_i)\|^2_2\,,
\end{equation}
where $\vf_t(\vx_i)$ and $\vf_s(\vx_i)$ represent the teacher and student features for the $i$-th image in the training set, and $D$ denotes the dimensionality of the feature (after the linear embedding $fc1_s$). Note that $\Loss_{mse}$ addresses both feature direction and magnitude. On the contrary, we propose to use locality-sensitive hashing (LSH)~\cite{LSH} to give the student more freedom with regard to its magnitude, but let the student concentrate more on mimicking the feature direction.

\begin{figure}[t]
	\centering
	\includegraphics[width=0.8\linewidth]{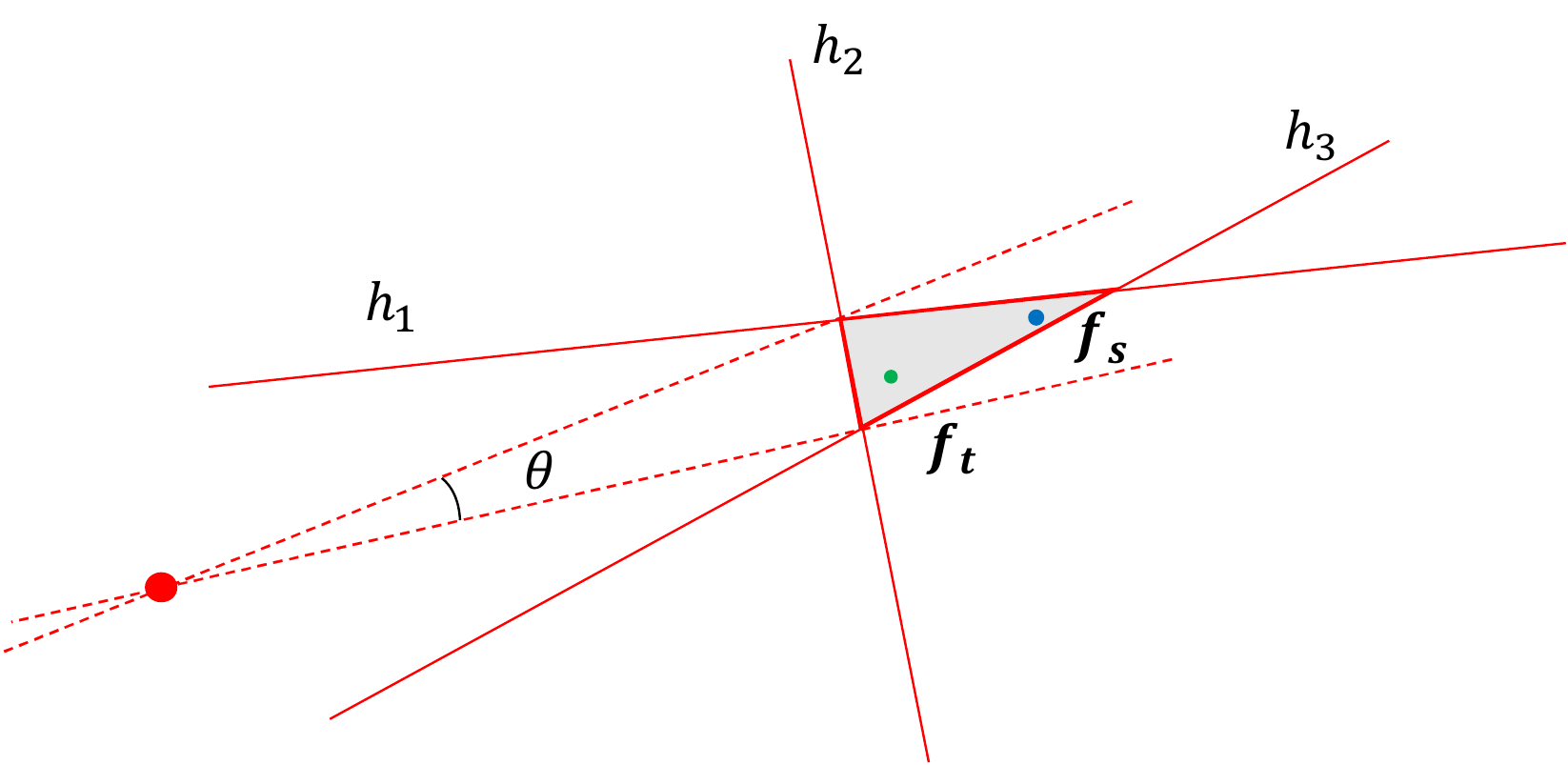}
	\caption{An illustration of the LSH loss. $\vf_s$ and $\vf_t$ denote a student and a teacher feature vector for the same input image, respectively. $h$ represents the hash function constraints. These constraints form a small polyhedron (the shaded region) and $\theta$ is the maximum angle between any two vectors in this polyhedron. Magnitudes of these features, however, can alter with greater freedom.}
	\label{fig:lsh}
\end{figure}

Figure~\ref{fig:lsh} shows an illustration for $\Loss_{lsh}$. In our LSH module, each hash function can be considered as a linear constraint. Many constraints will divide the feature space into a lot of polyhedra, and in general each polyhedron will be small. The LSH loss will encourage $\vf_s$ and $\vf_t$ to fall into the same polyhedron. Hence, more hash functions will result in smaller polyhedra, which in turn means that the angle between $\vf_t$ and $\vf_s$ will be small (upper bounded by $\theta$ in Figure~\ref{fig:lsh}, which is small itself because the polyhedron is small compared to the feature magnitudes.) In short, the LSH module encourages $\vf_t$ and $\vf_s$ to have similar directions, but relaxes constraints on their magnitudes. 

LSH aims at hashing the points into bins by several hash functions to ensure that, for each function, near points will fall into the same bin with high probability. In our framework, we use the hash family based on the Gaussian distribution which is a 2-stable distribution, defined as
\begin{equation}
	\label{eq:hash-func}
	h_{\vw, b}(\vf) = \left\lfloor \frac{\vw\tran\vf+b}{r}\right\rfloor\,,
\end{equation}
where $\vf\in\R^{D}$ is the feature, $\vw\in\R^{D}$ is a random vector whose entries are sampled from a Guassian distribution, $b$ is a real number chosen uniformly from the range $[0, r]$, $r$ is the length of each bin, and $\lfloor \cdot \rfloor$ is the floor function.

Our loss term $\Loss_{lsh}$ encourages the student feature to fall into the same bin as that of the teacher's. According to the theory of locality-sensitive hashing, for two vectors $\vf_1$, $\vf_2$, the probability of collision decreases monotonically with the distance between $\vf_1$ and $\vf_2$. Therefore, $h_{\mW, b}(\vf_t)=h_{\mW, b}(\vf_s)$ (which will result in a low value of $\Loss_{lsh}$) is a necessary condition for $\| \vf_t - \vf_s\|_2=0$. Hence, it is reasonable to force the student to mimic the teacher by minimizing $\Loss_{lsh}$. 

In our framework, we use $N$ hash functions with the form in Equation~\ref{eq:hash-func}. The locality-sensitive hashing module will generate $N$ hash codes for each feature. $0$ is used as the threshold to chop the real line. Therefore, the LSH module can be implemented by a FC layer and a signum function:
\begin{equation}
	\label{eq:hash-func-t}
	h_{\mW, \vb}(\vf) = \sign(\mW\tran\vf+\vb)\,,
\end{equation}
\begin{equation}
	\sign(x) = \begin{cases}
		1, & \text{if } x > 0;\\
		0, & \text{otherwise}\,,
	\end{cases}
\end{equation}
in which $\vf\in\R^{D}$ is the feature, $\mW\in\R^{D\times N}$ is the weights whose entries are sampled from a Guassian distribution and $\vb\in\R^{N}$ is the bias. Equation~\ref{eq:hash-func-t} generates $N$ binary codes for teacher feature $\vf_t$, and the hash code for student feature is expected to be the same as teacher's. We \emph{enforce this requirement by learning a classification problem}, and the binary cross entropy loss is to be minimized, i.e.,
\begin{equation}
	\mathbf{h} = \sign\left(\mW\tran\vf_t+\vb\right)\,,
\end{equation}
\begin{equation}
	\mathbf{p} = \sigma\left(\mW\tran\vf_s+\vb\right)\,,
\end{equation}
\begin{equation}
	\Loss_{lsh} = -\frac{1}{nN}\sum_{i=1}^{n}\sum_{j=1}^{N}\left[h_j\log p_j+(1-h_j)\log(1-p_j)\right]\,,
\end{equation}
where $\sigma$ is the sigmoid function $\sigma(x)=\frac{1}{1+\exp (-x)}$, $h_j$ and $p_j$ are the $j$-th entry of $\mathbf{h}$ and $\mathbf{p}$, respectively. 

Finally, inspired by \cite{supervised_kd}, we only distill the features which the teacher classifies correctly. To reduce the effect of randomness in the locality-sensitive hashing module, the average of the last 10 epochs' models during training is used as our final model.

\subsection{Experimental analysis}
\label{sec:analysis}

\begin{table}
	\caption{The difference between teacher features and student features. The  statistics were  estimated average values on the training and testing sets of CIFAR-100. $\|\vf_t\|_2$ and $\|\vf_s\|_2$ denote the 2-norm of teacher and student features, respectively. $\theta$ represents the average angle between them. Acc@1 is the accuracy (\%) of the student model.}
	\label{tab:CIFAR-100-only-feat}
	\centering
	\setlength{\tabcolsep}{2pt}
	\begin{tabular}{c|c|SS|SS}
		\toprule
		& Teacher & \multicolumn{2}{c}{vgg13} & \multicolumn{2}{c}{ResNet50}  \\
		& Student & \multicolumn{2}{c}{vgg8} & \multicolumn{2}{c}{MobileNetV2}  \\
		& dataset & {train} & {test} & {train} & {test} \\
		\midrule
		& $\|\vf_t\|_2$  & 12.64 & 11.83 & 15.52 & 15.03  \\
		\midrule
		\multirow{3}*{CE}            
		& $\|\vf_s\|_2$	& 16.50 & 16.16 & 16.64 & 16.33 \\
		& $\theta$ 					& 69.49\si{\degree} & 68.00\si{\degree} & 90.09\si{\degree} & 90.07\si{\degree} \\
		& Acc@1						& 99.19 & 70.72 & 90.31 & 64.36 \\
		\midrule
		\multirow{3}*{$\ell_2$ + CE}    
		& $\|\vf_s\|_2$	& 12.53 & 12.06 & 13.82 & 13.59 \\
		& $\theta$ 					& 26.86\si{\degree} & 29.90\si{\degree} & 32.04\si{\degree} & 33.25\si{\degree} \\
		& Acc@1 					& 98.26 & 72.33 & 89.07 & 65.73 \\
		\midrule
		\multirow{3}*{LSH + CE}   
		& $\|\vf_s\|_2$	& 24.74 & 23.73 & 9.69 & 9.60 \\
		& $\theta$      			& 28.22\si{\degree} & 31.00\si{\degree} & 31.28\si{\degree} & 32.29\si{\degree}  \\
		& Acc@1 					& 97.60 & 72.69 & 84.11 & 67.02  \\
		\midrule
		\multirow{3}*{LSH+$\ell_2$+CE} 
		& $\|\vf_s\|_2$	& 15.22 & 14.50 &  9.94 & 9.83  \\
		& $\theta$  				& 25.43\si{\degree} & 28.99\si{\degree} & 29.73\si{\degree} & 30.80\si{\degree} \\
		& Acc@1 					& 97.72 & 73.68 &  85.76 & 68.99  \\
		\bottomrule
	\end{tabular}
\end{table}

We use experiments to demonstrate the advantage of giving the student more freedom to the feature magnitude and making it focus on mimicking the feature direction. 

Table~\ref{tab:CIFAR-100-only-feat} shows the experimental results. The models vgg13 and vgg8 share similar architectures, while ResNet50 and MobileNetV2 have different architectures. ``CE'' denotes training the student by only the cross entropy loss without a teacher. We find that $\|\vf_s\|_2$ is very different from $\|\vf_t\|_2$. More statistics on $\|\vf\|_2$ of different models can be found in Table~\ref{tab:arch}. When knowledge distillation is not used, teacher and student features have very different directions as there are large angles between them, especially when their architectures are different.

When the $\ell_2$ loss ($\ell_2$ + CE) is used for feature mimicking, the student features are encouraged to be similar to the teacher features in \emph{both magnitudes and angles}, and the student accuracy is higher. 

The proposed LSH loss \emph{gives the student more freedom to its feature magnitude}. With the LSH loss (LSH + CE), vgg8 gets a larger feature magnitude while MobileNetV2 gets a smaller feature magnitude than that of CE. For vgg8, although $\theta$ of LSH + CE is a little larger than that of $\ell_2$ + CE, the accuracy of LSH + CE is higher, which shows the benefit of giving more freedom to the feature magnitude. For MobileNetV2, LSH + CE achieves both a smaller $\theta$ and better performance.

Finally, the LSH loss and the $\ell_2$ loss can be combined to help each other, and result in both smaller $\theta$ (i.e., similar directions) and better accuracy rates.

In Section~\ref{sec:theory}, we will analyze the LSH module theoretically. 

\subsection{Ensemble all losses}
\label{sec:Ensemble all loss}

The final loss consists of two terms, the classification and the feature mimicking losses. The regular cross-entropy loss $\Loss_c$ is used as the classification loss. We use both $\Loss_{mse}$ and $\Loss_{lsh}$ as the feature mimicking loss. Different from CRD~\cite{CRD} and SSKD~\cite{SSKD}, our method does \emph{not} need the knowledge distillation loss~\cite{KD} (KL-divergence between teacher and student logits with temperature). The final loss is 
\begin{equation}
	\label{eq:final_loss}
	\Loss=\Loss_c + \beta (\Loss_{mse} + \Loss_{lsh})\,,
\end{equation}
where $\beta$ is the balancing weight. Therefore, if the mean square loss is already used in other researches (e.g., detection, segmentation), our LSH module can be added directly without introducing extra hyperparameter. 

\subsection{Model initialization}
\label{sec:Model initialization}

The LSH module needs to be initialized before the end-to-end training. In the LSH module, the entries of $\mW$ are sampled from a Guassian distribution. We always set $0$ as its mean and treat the standard deviation ($std_{hash}$) as a hyperparameter. To find a good default value for $std_{hash}$, we collect statistics about the standard deviation ($std$) of the final classifier's weight ($\mathbf{W'}$) with vanilla training (cf. Table~\ref{tab:arch}). Assume $\mathbf{W'}=[\mathbf{w'}_1, \mathbf{w'}_2, \cdots, \mathbf{w'}_c]^{\mathsf{T}}$, where $\mathbf{w'}_i\in\R^{D}$ and $c$ is the number of categories, the expectation of $\|\mathbf{w'}\|_2$ can be roughly calculated by
\begin{equation}
	E(\|\mathbf{w'}\|_2)=E(\sqrt{\mathbf{w'}\tran\mathbf{w'}})\approx std\times\sqrt{D}\,,
\end{equation}
where the last transition holds because we noticed that the mean of $\mathbf{W'}$ is roughly zero. There is a tendency that $E(\|\mathbf{w'}\|_2)$ does not change drastically, and $std$ will become small when $D$ is large. These phenomena inspire us to choose $std_{hash}$ according to $D$. We also find that directly using the $std$ of teacher's final classifier's weight is a good default value for $std_{hash}$.

\begin{figure}
	\centering 
	\subfloat[]{\label{fig:sub1}\includegraphics[width=0.3\linewidth]{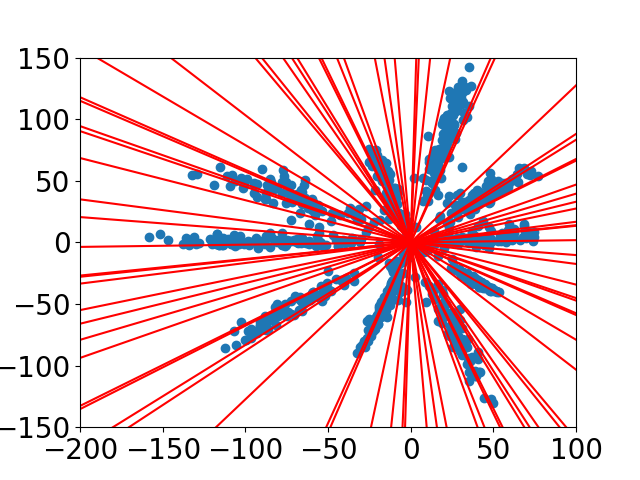}}
	\subfloat[]{\label{fig:sub2}\includegraphics[width=0.3\linewidth]{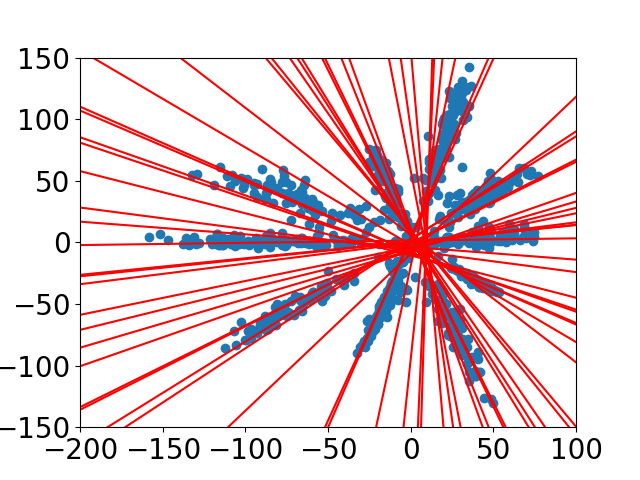}}
	\subfloat[]{\label{fig:sub3}\includegraphics[width=0.3\linewidth]{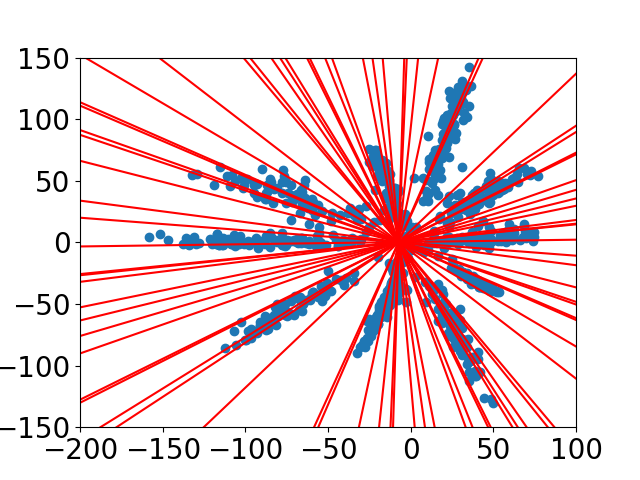}}\\ 
	\caption{Illustration of different initialization for the bias. Red lines denote the hash function constraints, while blue points represent teacher features. \protect\subref{fig:sub1}, \protect\subref{fig:sub2}, and \protect\subref{fig:sub3} show the bias initialized by $\mathbf{0}$, the median, or the mean of the teacher hash values, respectively. This figure is best viewed in color and zoomed in.}
	\label{fig:bias} 
\end{figure}

$\vb$ is the bias in the LSH module. As shown in Figure~\ref{fig:bias}, the bias can be initialized by  $\mathbf{0}$, the median, or the mean of the teacher hash values. Because BCE loss is applied, to make the binary classification problem balanced, we use the median of the teacher hash values as the bias in our LSH module. We also tried to use the mean of the teacher hash values or simply set $\vb=\mathbf{0}$. Later we will exhibit in Table~\ref{tab:CIFAR-100-same-arch-bias} and Table~\ref{tab:CIFAR-100-diff-arch-bias} the experimental results when using different initialization for the bias. These results show that our method is not sensitive to the initialization of $\vb$.

\section{Theoretical Analyses}
\label{sec:theory}

Now we will analyze why the LSH loss is sensitive to the feature's direction but not to the feature's magnitude. First, the following Claim~\ref{claim:1} says that if the teacher features are scaled, $\Loss_{lsh}$ will not change, i.e., $\Loss_{lsh}$ is not sensitive to the teacher feature's magnitude.

\begin{claim}
	\label{claim:1}
	For a given scale $s>0$, $\Loss_{lsh}(s\vf_t, \vf_s)=\Loss_{lsh}(\vf_t, \vf_s)$ for arbitrary $\vf_s$.
\end{claim}

Next, the following Claim~\ref{claim:2} states that when $\vf_s$ and $\vf_t$ have the same direction, $\Loss_{lsh}$ will encourage $\vf_s$ to be longer. 

\begin{claim}
	\label{claim:2}
	Assume the direction of $\vf_s$ is the same as that of $\vf_t$, and $\vb=\mathbf{0}$ in LSH. For a given scale $s>1$, then $\Loss_{lsh}(\vf_t, s\vf_s)\leq \Loss_{lsh}(\vf_t, \vf_s)$ always holds. 
\end{claim}

Finally, the following Claim~\ref{claim:lsh_p} and Claim~\ref{claim:degree_p} are the most important conclusions, which explain why our LSH loss can help the student to mimic the direction of teacher features. Claim~\ref{claim:lsh_p} computes the probability of the LSH loss being small (less than $\log{2}$) when we are given $\anglefn{{(\vf_t, \vf_s)}}$, the angle between teacher and student features. Hence, if the angle between $\vf_s$ and $\vf_t$ is smaller, the LSH loss will become small with higher probability.

Claim~\ref{claim:degree_p} gives the probability of $\anglefn{{(\vf_t, \vf_s)}} < \epsilon$ under the constraint that $\Loss_{lsh}$ is small. Using the probability formula in Claim~\ref{claim:degree_p}, we can numerically calculate the cumulative probability of the angle when $\Loss_{lsh}$ meets the condition (cf. Figure~\ref{fig:lsh_angle}). From this figure, we can conclude that if more hashing functions are used, the direction of $\vf_s$ will approach that of $\vf_t$ with higher probability.

\begin{claim}
	\label{claim:lsh_p}
	Suppose $\vb=\mathbf{0}$ in LSH, and $\vf_s$ and $\vf_t$ follow the standard Gaussian distribution. Then,
	\begin{equation}
		\prob{l_j < \log{2} \mid \anglefn{{(\vf_t, \vf_s)}} = \theta} = 1 - \frac{\theta}{\pi}
	\end{equation}
	will hold, where $\anglefn{{(\vf_t, \vf_s)}}$ denotes the angle between $\vf_t$ and $\vf_s$, and 
	\begin{equation}
		l_j \doteq -h_j\log{\left(p_j\right)} - \left(1 - h_j\right)\log{\left(1 - p_j\right)}\,.
	\end{equation}
\end{claim}

\begin{claim}
	\label{claim:degree_p}
	Suppose $\vb=\mathbf{0}$ in LSH, and $\vf_s$ and $\vf_t$ follow the standard Gaussian distribution.  Then, for any $0 < \epsilon < \pi$, the equation 
	\begin{align}
		&\prob{\anglefn{{(\vf_t, \vf_s)}} < \epsilon \mid \bigwedge_{j = 1}^{N}{\left(l_j < \log{2}\right)}} \notag \\
		=&\frac{\int_{0}^{\epsilon}{{\left(\left(1 - \frac{\theta}{\pi}\right)^{N} \cdot \sin^{D-2}{(\theta)} \right)} \,\mathrm{d}\theta}}{\int_{0}^{\pi}{{\left(\left(1 - \frac{\theta}{\pi}\right)^{N} \cdot \sin^{D-2}{(\theta)} \right)} \,\mathrm{d}\theta}}
	\end{align}
	will hold.
\end{claim}

\begin{figure}[t]
	\centering
	\includegraphics[width=0.8\linewidth]{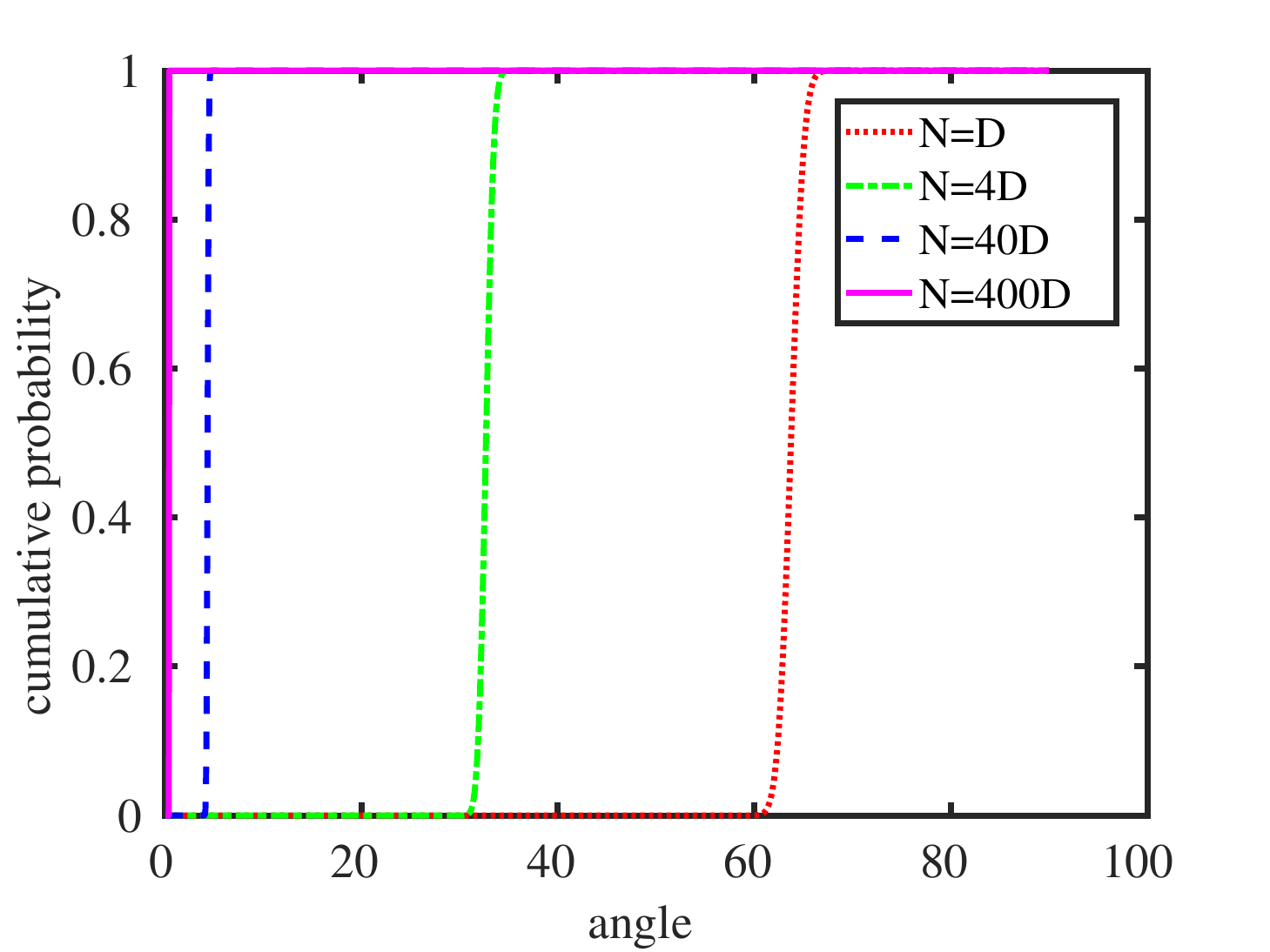}
	\caption{The cumulative probability of the angle with $D=2048$, where $D$ and $N$ denote the feature dimensionality and the number of hash functions, respectively. With $N$ becoming larger, the angle between the student feature and the teacher feature will become smaller with high probability. This figure is best viewed in color and zoomed in.}
	\label{fig:lsh_angle}
\end{figure}

The proof of Claim~\ref{claim:lsh_p} and Claim~\ref{claim:degree_p} are provided in the appendix to this paper. Utilizing these results, we numerically calculate the probability in Claim~\ref{claim:degree_p} for different $N$ values. As Figure~\ref{fig:lsh_angle} shows, when the number of hash function $N$ grows, the angle between teacher and student features indeed converge to 0. That is, our LSH loss is effective in mimicking the teacher feature's direction.

\section{Experiments}
\label{sec:experiments}

\begin{table}
	\caption{The networks used in our experiments. ResNet34 and ResNet18 were used on ImageNet, while other models were used on CIFAR-100. $D$ denotes the dimensionality of the feature before the final classifier. $std$ represents the standard deviation of the final classifier's weight with vanilla training. $\|\vf\|_2$ is the mean of the 2-norm of features in the training set. }
	\label{tab:arch}
	\centering
	\begin{tabular}{c|cccc}
		\toprule
		role & model & $D$ & $std$ & $\|\vf\|_2$\\
		\midrule
		\multirow{6}*{Teacher} & WRN-40-2 & 128 & 0.1713 & 13.64\\
		& resnet56 & 64 & 0.2415 & 18.08 \\
		& resnet110 & 64 & 0.2262 & 20.13 \\
		& resnet32x4 & 256 & 0.1100 & 12.06 \\
		& vgg13 & 512 & 0.0749 & 12.64 \\
		& ResNet50 & 2048 & 0.0381 & 15.52\\
		\midrule
		\multirow{9}*{Student} & WRN-16-2 & 128 & 0.2035 & 15.07\\
		& WRN-40-1 & 64 & 0.2638 & 14.29 \\
		& resnet20 & 64 & 0.2704 & 14.47 \\
		& resnet32 & 64 & 0.2563 & 16.06 \\
		& resnet8x4 & 256 & 0.1573 & 16.92 \\
		& vgg8 & 512 & 0.0980 & 16.50 \\
		& MobileNetV2 & 640 & 0.0579 & 15.82\\
		& ShuffleNetV1 & 960 & 0.0642 & 17.26\\
		& ShuffleNetV2 & 1024 & 0.0625 & 15.70 \\
		\midrule
		Teacher & ResNet34 & 512 & 0.0640 & 30.75 \\
		Student & ResNet18 & 512 & 0.0695 & 29.58 \\
		\bottomrule
	\end{tabular}
\end{table}

In this section, we evaluate the proposed feature mimicking framework on single-label classification, multi-label recognition, and object detection. For single-label classification, we use the CIFAR-100~\cite{cifar} and  ImageNet~\cite{imagenet} datasets, which are usually used as benchmarks for knowledge distillation. CIFAR-100 contains $32\times 32$ natural images from 100 categories, which contains 50000 training images and 10000 testing images. ImageNet is a large-scale dataset with natural color images from 1000 categories. Each category typically has 1300 images for training and 50 for evaluation. 

For CIFAR-100, we used the code provided by CRD~\cite{CRD}.\footnote {\url{https://github.com/HobbitLong/RepDistiller}} For a fair comparison, we used the same hyperparameters of CRD in our experiments, such as the learning rate, batch size and epoch. For ImageNet, we followed the standard PyTorch example code and trained 100 epochs (following CRD).\footnote{\url{https://github.com/pytorch/examples/tree/master/imagenet}}

\subsection{Ablation studies}

We first study the  effects of the loss functions, hyperparameters in LSH, and model initialization.

\subsubsection{The loss functions}

First, we conduct ablation studies on the loss functions. Our final loss contains $\Loss_{mse}$ and $\Loss_{lsh}$. We will use only one of them to see their individual effects. 

\begin{table*}
	\caption{Test accuracy (\%) of the student network on CIFAR-100. The teacher and the student share similar architectures.}
	\label{tab:CIFAR-100-same-arch-lshloss}
	\centering
	\begin{tabular}{c|ccccccc}
		\toprule
		Teacher & WRN-40-2 & WRN-40-2 & resnet56 & resnet110 & resnet110 & resnet32x4 & vgg13 \\
		Student & WRN-16-2 & WRN-40-1 & resnet20 & resnet20 & resnet32 & resnet8x4 & vgg8 \\
		\midrule
		Teacher & 75.61 & 75.61 & 72.34 & 74.31 & 74.31 & 79.42 & 74.64 \\
		Student & 73.26 & 71.98 & 69.06 & 69.06 & 71.14 & 72.50 & 70.36 \\
		\midrule
		KD              & $74.92\,(71\%\uparrow)$ & $73.54\,(43\%\uparrow)$ & $70.66\,(49\%\uparrow)$ & $70.67\,(31\%\uparrow)$ & $73.08\,(61\%\uparrow)$ & $73.33\,(12\%\uparrow)$ & $72.98\,(61\%\uparrow)$ \\
		$\ell_2$ loss & $75.53\,(97\%\uparrow)$ & $74.33\,(65\%\uparrow)$ & $71.42\,(72\%\uparrow)$ & $71.30\,(43\%\uparrow)$ & $73.81\,(84\%\uparrow)$ & $74.01\,(22\%\uparrow)$ & $72.33\,(46\%\uparrow)$ \\
		LSH loss & $75.61\,(100\%\uparrow)$ & $74.20\,(61\%\uparrow)$ & $71.51\,(75\%\uparrow)$ & $71.73\,(51\%\uparrow)$ & $73.69\,(80\%\uparrow)$ & $73.49\,(14\%\uparrow)$ & $72.69\,(54\%\uparrow)$ \\
		$\ell_2$ loss + LSH loss & $75.62\,(100\%\uparrow)$ & $74.54\,(71\%\uparrow)$ & $71.65\,(79\%\uparrow)$ & $71.39\,(44\%\uparrow)$ & $73.99\,(90\%\uparrow)$ & $73.37\,(13\%\uparrow)$ & $73.68\,(78\%\uparrow)$ \\
		\bottomrule
	\end{tabular}
\end{table*}

\begin{table*}[!htb] 
	\caption{Test accuracy (\%) of the student network on CIFAR-100. The teacher and the student use different architectures. }
	\label{tab:CIFAR-100-diff-arch-lshloss}
	\centering
	\setlength{\tabcolsep}{4pt}
	\begin{tabular}{c|cccccc}
		\toprule
		Teacher & vgg13 & ResNet50 & ResNet50 & resnet32x4 & resnet32x4 & WRN-40-2 \\
		Student & MobileNetV2 & MobileNetV2 & vgg8 & ShuffleNetV1 & ShuffleNetV2 & ShuffleNetV1 \\
		\midrule
		Teacher & 74.64 & 79.34 & 79.34 & 79.42 & 79.42 & 75.61 \\
		Student & 64.60 & 64.60 & 70.36 & 70.50 & 71.82 & 70.50 \\
		\midrule
		KD           & $67.37\,(28\%\uparrow)$ & $67.35\,(19\%\uparrow)$ & $73.81\,(38\%\uparrow)$ & $74.07\,(40\%\uparrow)$ & $74.45\,(35\%\uparrow)$ & $74.83\,(85\%\uparrow)$ \\
		$\ell_2$ loss & $66.98\,(24\%\uparrow)$ & $65.73\,(8\%\uparrow)$ & $71.90\,(17\%\uparrow)$ & $74.65\,(47\%\uparrow)$ & $75.73\,(51\%\uparrow)$ & $75.37\,(95\%\uparrow)$ \\
		LSH loss 	& $67.48\,(29\%\uparrow)$ & $67.02\,(16\%\uparrow)$ & $74.15\,(42\%\uparrow)$ & $75.49\,(56\%\uparrow)$ & $75.56\,(49\%\uparrow)$ & $75.89\,(105\%\uparrow)$ \\
		$\ell_2$ loss + LSH loss & $67.16\,(25\%\uparrow)$ & $68.99\,(30\%\uparrow)$ & $74.89\,(50\%\uparrow)$ & $75.36\,(54\%\uparrow)$ & $76.70\,(64\%\uparrow)$ & $76.25\,(113\%\uparrow)$ \\
		\bottomrule
	\end{tabular}
\end{table*}

Table~\ref{tab:CIFAR-100-same-arch-lshloss} and Table~\ref{tab:CIFAR-100-diff-arch-lshloss} summarize the results. We used ``KD''~\cite{KD} as the baseline method. Note that all experiments used the classification loss $\Loss_{c}$. ``$\ell_2$ loss'' denotes only using $\Loss_{mse}$, while ``LSH loss'' represent only using $\Loss_{lsh}$. ``$\ell_2$ loss + LSH loss'' combines the $\Loss_{mse}$ and $\Loss_{lsh}$ as in Equation~\ref{eq:final_loss}. To balance the feature mimicking loss and classification loss, $\beta$ was set as $6$. In these tables, we also show the relative improvement as a percentage. Accuracy of the student and the teacher are treated as 0\% and 100\%, respectively. For example, in the last column of Table~\ref{tab:CIFAR-100-diff-arch-lshloss}, the student and teacher accuracy are 70.50 and 75.61, while the proposed ``$\ell_2$ loss + LSH loss'' is 76.25, hence the relative improvement is $\frac{76.25-70.50}{75.61-70.50}=113\%$.

When the teacher and student share similar architectures, only using the $\ell_2$ loss can surpass the standard KD significantly, which demonstrates the advantage of feature mimicking for knowledge distillation. When only applying the LSH loss we proposed, the performance of most teacher/student combinations are better than that of the $\ell_2$ loss, showing the benefit of giving the student more freedom to the feature magnitude and letting it focus on mimicking the feature direction. Combining $\ell_2$ and LSH losses can boost the performance. We believe it is because the LSH loss can alleviate the shortcomings of the $\ell_2$ loss, and the LSH loss can also benefit from the $\ell_2$ loss. 

When the teacher and student use different architectures, the difference of their accuracy is larger than that in the  similar-architecture settings, and their features are more different. Due to the limited capacity of student networks, it is difficult for the student to mimic both features' directions and magnitudes. The experimental results in Table~\ref{tab:CIFAR-100-diff-arch-lshloss} show that only using $\Loss_{lsh}$ outperforms using $\Loss_{mse}$ in most cases, which justifies that feature directions have more effective information to boost the student performance, and that we should make the student pay more attention to the feature direction. Combining $\ell_2$ and LSH losses is consistently better than only applying the $\ell_2$ loss. It demonstrates that feature mimicking indeed benefits from giving more freedom to the student feature's magnitude.

Furthermore, by comparing the relative improvement numbers in Table~\ref{tab:CIFAR-100-same-arch-lshloss} and Table~\ref{tab:CIFAR-100-diff-arch-lshloss}, it is obvious that knowledge distillation across different network architectures is a more challenging task than distilling between similar-architecture networks. Hence, it is not surprising that differences among the $\ell_2$ loss, the proposed LSH loss, and the ``$\ell_2$ loss + LSH loss'' are relatively small in Table~\ref{tab:CIFAR-100-same-arch-lshloss}. On the other hand, Table~\ref{tab:CIFAR-100-diff-arch-lshloss} confirms that the proposed LSH loss is supervisor to the $\ell_2$ loss, which also shows that the combination of these two are complementary in feature mimicking. For example, when we distill knowledge from ResNet50 to MobileNetV2, the combined relative improvement (30\%) is even higher than the sum of both (8\% + 16\%).

\subsubsection{Hyperparameters in the LSH module}

Next, we study the effect of hyperparameters in the LSH loss. There are three hyperparameters in locality-sensitive hashing. $N$ denotes the number of hashing functions. $std_{hash}$ represents the standard deviation of the Gaussian sampler. Note that we always use $0$ as the mean of the Gaussian sampler. $\beta$ is the balancing weight for both $\Loss_{lsh}$ and $\Loss_{mse}$. 

\begin{table*}
	\caption{Test accuracy (\%) of the student network on CIFAR-100 using different hyperparameters ($\beta$, $std_{hash}$, $N$). The teacher and the student share similar architectures.}
	\label{tab:CIFAR-100-b-same-arch}
	\centering
	\begin{tabular}{c|ccccccc}
		\toprule
		Teacher & WRN-40-2 & WRN-40-2 & resnet56 & resnet110 & resnet110 & resnet32x4 & vgg13 \\
		Student & WRN-16-2 & WRN-40-1 & resnet20 & resnet20 & resnet32 & resnet8x4 & vgg8 \\
		\midrule
		$std_{t}$		& 0.17 & 0.17 & 0.24 & 0.23 & 0.23 & 0.11 & 0.07 \\
		$std_{s}$		& 0.20 & 0.26 & 0.27 & 0.27 & 0.26 & 0.16 & 0.10 \\
		$D_t$ & 128 & 128 & 64 & 64 & 64 & 256 & 512 \\
		$D_s$ & 128 & 64 & 64 & 64 & 64 & 256 & 512 \\
		\midrule
		$(1, 1, 2048)$ & 75.33 & 73.50 & 71.25 & 71.17 & 73.37 & 73.64 & 73.06 \\
		$(3, 1, 2048)$ & 75.47 & 74.16 & 71.70 & 71.68 & 73.87 & 74.11 & 72.91 \\
		$(5, 1, 2048)$ & 75.99 & 74.43 & 71.41 & 71.66 & 73.32 & 73.66 & 73.77 \\
		$(6, 1, 2048)$ & 75.62 & 74.54 & 71.65 & 71.39 & 73.99 & 73.37 & 73.68 \\
		$(7, 1, 2048)$ & 76.34 & 74.36 & 71.18 & 71.78 & 73.96 & 73.70 & 73.89 \\
		\midrule
		$(6, std_{t}, 2048)$ & 76.11 & 74.42 & 70.96 & 71.75 & 74.00 & 73.91 & 73.57 \\
		$(6, std_{s}, 2048)$ & 75.53 & 74.25 & 71.43 & 71.60 & 74.19 & 73.82 & 73.61 \\
		\midrule
		$(6, std_{t}, 4\times D_t)$ & 76.43 & 74.15 & 71.27 & 71.13 & 73.55 & 74.13 & 73.57 \\
		$(6, std_{t}, 32\times D_t)$ & 75.84 & 74.51 & 70.96 & 71.75 & 74.00 & 73.77 & 73.25 \\
		\bottomrule
	\end{tabular}
\end{table*}

\begin{table*}[!htb] 
	\caption{Test accuracy (\%) of the student network on CIFAR-100 using different hyperparameters ($\beta$, $std_{hash}$, $N$). The teacher and the student use different architectures.}
	\label{tab:CIFAR-100-b-diff-arch}
	\centering
	\begin{tabular}{c|cccccc}
		\toprule
		Teacher & vgg13 & ResNet50 & ResNet50 & resnet32x4 & resnet32x4 & WRN-40-2 \\
		Student & MobileNetV2 & MobileNetV2 & vgg8 & ShuffleNetV1 & ShuffleNetV2 & ShuffleNetV1 \\
		\midrule
		$std_{t}$		& 0.07 & 0.04 & 0.04 & 0.11 & 0.11 & 0.17 \\
		$std_{s}$		& 0.06 & 0.06 & 0.10 & 0.06 & 0.06 & 0.06 \\
		$D_t$ & 512 & 2048 & 2048 & 256 & 256 & 128 \\
		$D_s$ & 640 & 640 & 512 & 960 & 1024 & 960 \\
		\midrule
		$(1, 1, 2048)$ & 66.82 & 65.79 & 72.12 & 75.23 & 75.42 & 74.98 \\
		$(3, 1, 2048)$ & 67.95 & 67.33 & 73.47 & 74.94 & 76.12 & 76.17 \\
		$(5, 1, 2048)$ & 68.01 & 67.60 & 74.64 & 75.38 & 75.56 & 76.06 \\
		$(6, 1, 2048)$ & 67.16 & 68.99 & 74.89 & 75.36 & 76.70 & 76.25 \\
		$(7, 1, 2048)$ & 67.88 & 69.20 & 74.43 & 75.25 & 76.70 & 76.35 \\
		\midrule
		$(6, std_t, 2048)$      & 68.12 & 67.57 & 72.89 & 75.22 & 76.52 & 75.63 \\
		$(6, std_s, 2048)$      & 67.77 & 67.47 & 73.68 & 74.93 & 76.27 & 75.70 \\
		\midrule
		$(6, std_t, 4\times D_t)$ & 68.12 & 67.95 & 72.80 & 75.02 & 76.46 & 76.36 \\
		$(6, std_t, 32\times D_t)$ & 67.78 & 67.33 & 72.76 & 75.36 & 76.25 & 75.83 \\
		\bottomrule
	\end{tabular}
\end{table*}

Table~\ref{tab:CIFAR-100-b-same-arch} and Table~\ref{tab:CIFAR-100-b-diff-arch} summarize the results. First, when $std_{hash}=1$ and $N=2048$, different teacher/student combinations achieve the best results with different $\beta$. So it is better to use a validation set to tune this hyperparameter. Limited by computation resources, we simply used $\beta=6$ for all experiments on CIFAR-100. Second, the value of $std_{hash}$ also affect the performance. But we find that it is less sensitive than $\beta$. Third, a larger $N$ may reduce the randomness in LSH. Experiments show that setting $N=2048$ is good enough. Overall, if applying our method to other problems, we suggest that $N=2048$ or $N=4D_t$,  $std_{hash}=1$ or $std_{hash}=std_t$, and finally using a validation set to tune $\beta$.

\subsubsection{Different model initialization}

\begin{table*}
	\caption{Test accuracy (\%) of the student network on CIFAR-100 with different initializations of bias in the LSH module. The teacher and the student share similar architectures. \textbf{Bold} denotes the best results.}
	\label{tab:CIFAR-100-same-arch-bias}
	\centering
	\begin{tabular}{c|ccccccc}
		\toprule
		Teacher & WRN-40-2 & WRN-40-2 & resnet56 & resnet110 & resnet110 & resnet32x4 & vgg13 \\
		Student & WRN-16-2 & WRN-40-1 & resnet20 & resnet20 & resnet32 & resnet8x4 & vgg8 \\
		\midrule
		KD              & 74.92 & 73.54 & 70.66 & 70.67 & 73.08 & 73.33 & 72.98 \\
		0 				  & \textbf{76.04} & 74.46 & 71.16 & \textbf{71.79} & \textbf{74.18} & \textbf{73.70} & 73.92 \\
		mean 		  & 75.39 & 74.11 & 71.52 & 70.95 & 73.85 & 73.64 & \textbf{73.98} \\
		median 		 & 75.62 & \textbf{74.54} & \textbf{71.65} & 71.39 & 73.99 & 73.37 & 73.68 \\
		\bottomrule
	\end{tabular}
\end{table*}

\begin{table*}[!htb] 
	\caption{Test accuracy (\%) of the student network on CIFAR-100 with different initializations of bias in the LSH module. The teacher and student use different architectures. \textbf{Bold} denotes the best results.}
	\label{tab:CIFAR-100-diff-arch-bias}
	\centering
	\begin{tabular}{c|cccccc}
		\toprule
		Teacher & vgg13 & ResNet50 & ResNet50 & resnet32x4 & resnet32x4 & WRN-40-2 \\
		Student & MobileNetV2 & MobileNetV2 & vgg8 & ShuffleNetV1 & ShuffleNetV2 & ShuffleNetV1 \\
		\midrule
		KD                   & 67.37 & 67.35 & 73.81 & 74.07 & 74.45 & 74.83 \\
		0 					  & 67.14 & 68.64 & 74.25 & \textbf{75.57} & \textbf{76.71} & 75.76 \\
		mean 			& \textbf{68.16} & 68.07 & 74.54 & 75.55 & 75.32 & 75.99 \\
		median 			& 67.16 & \textbf{68.99} & \textbf{74.89} & 75.36 & 76.70 & \textbf{76.25} \\
		\bottomrule
	\end{tabular}
\end{table*}

We study different initialization for bias in the LSH module. By default, the bias is initialized as the median of teacher hashing codes to balance the binary classification problem. We also tried to use the mean of teacher hashing codes or $\mathbf{0}$ to initialize the bias. Table~\ref{tab:CIFAR-100-same-arch-bias} and Table~\ref{tab:CIFAR-100-diff-arch-bias} present the results. We find that knowledge distillation is not sensitive to the initialization of bias. When apply our method on large-scale datasets (like ImageNet), we used $\mathbf{0}$ to initialize the bias because it is difficult to compute the median.

\subsection{Single-label Classification}
\label{sec:classification}

\begin{table*}
	\caption{Test accuracy (\%) of the student network on CIFAR-100. The teacher and the student share similar architectures. We denote by * methods where we re-run three times using author-provided code. And the results of our method were run by five times. \textbf{Bold} denotes the best results.}
	\label{tab:CIFAR-100-same-arch-sota}
	\centering
	\begin{tabular}{c|SSSSSSS}
		\toprule
		Teacher & {WRN-40-2} & {WRN-40-2} & {resnet56} & {resnet110} & {resnet110} & {resnet32x4} & {vgg13} \\
		Student & {WRN-16-2} & {WRN-40-1} & {resnet20} & {resnet20} & {resnet32} & {resnet8x4} & {vgg8} \\
		\midrule
		Teacher & 75.61 & 75.61 & 72.34 & 74.31 & 74.31 & 79.42 & 74.64 \\
		Student & 73.26 & 71.98 & 69.06 & 69.06 & 71.14 & 72.50 & 70.36 \\
		\midrule
		KD~\cite{KD}   			& 74.92 & 73.54 & 70.66 & 70.67 & 73.08 & 73.33 & 72.98 \\
		FitNet~\cite{fitnet}      & 73.58 & 72.24 & 69.21 & 68.99 & 71.06 & 73.50 & 71.02 \\
		AT~\cite{AT}           & 74.08 & 72.77 & 70.55 & 70.22 & 72.31 & 73.44 & 71.43 \\
		SP~\cite{SP}           & 73.83 & 72.43 & 69.67 & 70.04 & 72.69 & 72.94 & 72.68 \\
		AB~\cite{AB}           & 72.50 & 72.38 & 69.47 & 69.53 & 70.98 & 73.17 & 70.94 \\
		FT~\cite{FT}           & 73.25 & 71.59 & 69.84 & 70.22 & 72.37 & 72.86 & 70.58 \\
		FSP~\cite{FSP}         & 72.91 & {n/a} & 69.65 & 70.11 & 71.89 & 72.62 & 70.23 \\
		CRD~\cite{CRD}         & 75.48 & 74.14 & 71.16 & 71.46 & 73.48 & 75.51 & 73.94  \\
		CRD+KD~\cite{CRD} & 75.64 & 74.38 & \textbf{71.63} & 71.56 & 73.75 & 75.46 & 74.29 \\ 
		\rowcolor{Shade}
		Ours (1FC)  			& 75.99 & {-} & 71.39 & \textbf{71.64} & \textbf{73.90} & 73.40 & 73.78 \\ 
		\rowcolor{Shade}
		Ours  & \textbf{76.41} & 74.64 & 71.44 & 71.48 & 73.59 & 76.75 & 74.63 \\ 
		\midrule
		SSKD*~\cite{SSKD} & 75.55 & 75.50 & 71.00 & 71.27 & 73.60 & 76.13 & 74.90 \\
		\rowcolor{Shade}
		Ours + SSKD 		& 75.89 & \textbf{75.72} & 71.29 & 71.34 & 73.68 & \textbf{76.95} & \textbf{75.19} \\ 
		\bottomrule
	\end{tabular}
\end{table*}

\begin{table*}
	\caption{Test accuracy (\%) of the student network on CIFAR-100. The architectures of teacher and student are different. We denote by * methods where we re-run three times using author-provided code. And the results of our method were run by five times. \textbf{Bold} denotes the best results.}
	\label{tab:CIFAR-100-diff-arch-soa}
	
	\centering
	\begin{tabular}{c|SSSSSS}
		\toprule
		Teacher & {vgg13} & {ResNet50} & {ResNet50} & {resnet32x4} & {resnet32x4} & {WRN-40-2}  \\
		Student & {MobileNetV2} & {MobileNetV2} & {vgg8} & {ShuffleNetV1} & {ShuffleNetV2} & {ShuffleNetV1}  \\
		\midrule
		Teacher & 74.64 & 79.34 & 79.34 & 79.42 & 79.42 & 75.61 \\
		Student & 64.60 & 64.60 & 70.36 & 70.50 & 71.82 & 70.50 \\
		\midrule
		KD~\cite{KD}           		& 67.37 & 67.35 & 73.81 & 74.07 & 74.45 & 74.83 \\
		FitNet~\cite{fitnet}      & 64.14 & 63.16 & 70.69 & 73.59 & 73.54 & 73.73 \\
		AT~\cite{AT}           & 59.40 & 58.58 & 71.84 & 71.73 & 72.73 & 73.32 \\
		SP~\cite{SP}           & 66.30 & 68.08 & 73.34 & 73.48 & 74.56 & 74.52 \\
		AB~\cite{AB}           & 66.06 & 67.20 & 70.65 & 73.55 & 74.31 & 73.34 \\
		FT~\cite{FT}           & 61.78 & 60.99 & 70.29 & 71.75 & 72.50 & 72.03 \\
		CRD~\cite{CRD}         & 69.73 & 69.11 & 74.30 & 75.11 & 75.65 & 76.05  \\
		CRD+KD~\cite{CRD} & 69.94 & 69.54 & 74.58 & 75.12 & 76.05 & 76.27 \\ 
		\rowcolor{Shade} 
		Ours 		& 69.42 & 69.64 & 74.74 & 77.06 & 77.08 & \textbf{77.57} \\ 
		\midrule
		SSKD*~\cite{SSKD} & 71.24 & 71.81 & 75.71 & 78.18 & 78.75 & 77.30 \\
		\rowcolor{Shade} Ours+SSKD &  \textbf{71.77} & \textbf{72.38}& \textbf{76.13} & \textbf{78.32} & \textbf{79.01} & 77.46 \\
		\bottomrule
	\end{tabular}
\end{table*}

\begin{table*}[!htb] 
	\caption{Top-1 and Top-5 error rates (\%) on the ImageNet validation set. The teacher and student are ResNet-34 and ResNet-18, respectively. \textbf{Bold} denotes the best results.}
	\label{tab:imagenet}
	\centering
	\setlength{\tabcolsep}{3pt}
	\begin{tabular}{c|cc|cccccccc|cc}
		\toprule
		& Teacher & Student & CC~\cite{CC} & SP~\cite{SP} & Online-KD~\cite{OnlineKD} & KD~\cite{KD} & AT~\cite{AT} & CRD~\cite{CRD} & CRD+KD & SSKD~\cite{SSKD} & Ours ($\ell_2$) & Ours ($\ell_2$ + LSH) \\
		\midrule
		Top-1 & 26.70 & 30.25 & 30.04 & 29.38 & 29.45 & 29.34 & 29.30 & 28.83 & 28.62 & 28.38 & 28.61 & \textbf{28.28} \\
		Top-5 & 8.58  & 10.93 & 10.83 & 10.20 & 10.41 & 10.12 & 10.00 & 9.87 & 9.51 & \textbf{9.33} & 9.61 & 9.59 \\
		\bottomrule
	\end{tabular}
\end{table*}

Table~\ref{tab:CIFAR-100-same-arch-sota} and Table~\ref{tab:CIFAR-100-diff-arch-soa} compare our method with other knowledge distillation approaches on the CIFAR-100 benchmark. We simply set $\beta=6$, $std_{hash}=1$ and $N=2048$ for all experiments. And for a fair comparison, we used the same teacher networks as CRD~\cite{CRD}. Different from SSKD~\cite{SSKD}, we only used self-supervised learning~\cite{simclr} to train student networks and got the backbone weights to initialize our framework. 

Table~\ref{tab:CIFAR-100-same-arch-sota} presents the results when the teacher and student share similar architecture. Note that ``Ours (1FC)'' removed the linear embedding layer, which is possible because teacher and student features have the same dimensionality. Our method surpasses CRD+KD~\cite{CRD} on most teacher/student combinations. Note that our method did not use the original KD~\cite{KD} loss, and is thus more flexible. Compared with SSKD~\cite{SSKD}, our method outperform on five teacher/student combinations. And our method can be combined with SSKD (``Ours + SSKD''), which consistently outperforms SSKD. We simply set $\beta=0.01, std_{hash}=1, N=2048$ and added our loss terms into the SSKD framework.

Table~\ref{tab:CIFAR-100-diff-arch-soa} summarizes the results when the architectures of teacher and student are different. Our method outperformed CRD+KD~\cite{CRD} on the majority of teacher/student combinations, but slightly worse than SSKD~\cite{SSKD}. These results suggest that with different teacher/student architectures, self-supervised learning is critical for KD (because SSKD outperformed other methods). However, note that our method can be combined with SSKD, which consistently outperforms SSKD. Same as that on similar architecture, we simply set $\beta=0.01, std_{hash}=1, N=2048$ and added our loss terms into the SSKD framework. 

Table~\ref{tab:imagenet} summarize the results on ImageNet. The hyperparameters in our method are $\beta=5$, $std_{hash}=std_t$ and $N=2048$. Note that different from $CRD+KD$ and $SSKD$, we did not use the standard KD loss~\cite{KD} to boost the performance. Only using the $\ell_2$ loss to force the student features to mimic the teacher features outperforms CRD, which once again supports the validity of our proposed feature mimicking. Combining the $\ell_2$ and LSH losses further boosts the performance by a significant margin and achieves the state-of-the-art performance, which further supports the proposed LSH loss.

\subsection{Multi-label classification}
\label{sec:multi-label classification}

We consider two typical multi-label classification tasks, i.e., VOC2007~\cite{VOC} and MS-COCO~\cite{COCO}. VOC007 contains a train-val set of 5011 images and a test set of 4952 images. And MS-COCO contains 82081 images in the training set and 40137 images for validation. We resize all images into a fixed size ($448\times 448$) to train the networks. And the data augmentation consist of random horizontal flips and color jittering. The backbone networks contain MobileNetV2, ResNet18, ResNet34, ResNet50 and ResNet101. The networks are all pre-trained on ImageNet and finetuned on the multi-label classification dataset with stochastic gradient descent (SGD) for 60 epochs in total. The binary cross entropy (BCE) loss is used to finetune the network. We employ the mean average precision (mAP) to evaluate all the methods. Note that multi-label recognition is \emph{not a typical application of KD because existing KD methods rely on the soft logits, which do not exist in multi-label scenarios}. The proposed feature mimicking method, however, is \emph{flexible and handles multi-label distillation well}.

\begin{table}
	\caption{Test mAP (\%) on Pascal VOC2007.}
	\label{tab:VOC-ablation-study-1}
	\centering
	\begin{tabular}{l|cc}
		\toprule
		Teacher & ResNet34 & ResNet34 \\
		Student & ResNet18 & ResNet18 (pretrained by KD) \\
		\midrule
		Teacher    & 91.69 & 91.69 \\
		Student 	& 89.15 & 89.88 \\
		\midrule
		KD			& $89.26\,(4\%\uparrow)$ & $89.85\,(2\%\downarrow)$ \\
		$\ell_2$ (1FC)    & $88.75\,(20\%\downarrow)$ & $90.89\,(56\%\uparrow)$ \\
		$\ell_2$ (2FC)    & $89.98\,(33\%\uparrow)$ & $90.77\,(49\%\uparrow)$ \\
		\bottomrule
	\end{tabular}
\end{table}

First, we conduct experiments on VOC2007 using ResNet34 as teacher and ResNet18 as student to demonstrate that feature space alignment is necessary and important. The teacher ResNet34 is first trained on ImageNet and then finetuned on VOC2007. It achieves $91.69\%$ mAP as in Table~\ref{tab:VOC-ablation-study-1}. The student ResNet18 achieves $89.15\%$ mAP. And in Section~\ref{sec:classification}, we have trained ResNet18 supervised by ResNet34 on ImageNet. This model is denoted as ``ResNet18 (pretrained by KD)'' and achieves $89.88\%$ mAP. When finetuned on VOC2007 supervised by the teacher with the $\ell_2$ loss, ResNet18 achieves a worse performance ($88.75\%$) than baseline, which we believe is because the feature spaces of the teacher and student do not align well. If we use ResNet18 pretrained by KD whose feature space aligns to the teacher's, the student can be improved to $90.89\%$. With the 2FC structure, the first linear layer can transform the student feature space to align to the teacher's. It alleviates the feature space misalignment issue and achieve a better performance ($89.98\%$) than baseline. ResNet18 pretrained by KD with 2FC achieves a worse performance ($90.77\%$) than that with 1FC. That demonstrates it does not need the first linear layer to transform the feature space.

\begin{table}
	\caption{Test mAP (\%) of the student network on Pascal VOC07. \textbf{Bold} denotes the best results.}
	\label{tab:VOC-ablation-study-2}
	\centering
	\setlength{\tabcolsep}{3pt}
	\begin{tabular}{c|ccc}
		\toprule
		\diagbox{2nd stage}{1st stage} 	& L2 & LSH & LSHL2 \\
		\midrule
		L2          & $90.40\,(49\%\uparrow)$ & $90.21\,(42\%\uparrow)$ & $\textbf{90.59}\,(57\%\uparrow)$ \\
		LSH 		& $90.29\,(45\%\uparrow)$ & $90.11\,(38\%\uparrow)$ & $90.30\,(45\%\uparrow)$ \\
		LSHL2    	& $90.57\,(56\%\uparrow)$ & $90.37\,(48\%\uparrow)$ & $\textbf{90.59}\,(57\%\uparrow)$ \\
		\bottomrule
	\end{tabular}
\end{table}

Although the backbone pretrained by KD on a large scale dataset will transfer better and easily mimic the teacher's features during finetuning, it is expensive to pretrain the student on a large scale dataset in many cases. Hence, we propose a simple but effective approach to alleviate the feature space misalignment problem. We finetune the student by two stages. In the first stage, we fix the weights in the student backbone and only optimize the linear embedding layer with the feature mimicking loss functions. This stage aims at transform the student feature space to align to the teacher's. In the second stage, we add the classifier on top of the linear embedding layer and optimize all parameters in the student with the supervision of both the groundtruth labels and the teacher. Table~\ref{tab:VOC-ablation-study-2} summarize the results. With this two-stage training, the student can be improved by a large margin, compared with $89.98\%$ mAP when training the student by one stage. We find that the feature mimicking loss chosen in the first stage is important, and the LSHL2 ($\Loss_{mse} + \Loss_{lsh}$) loss is consistently better than the $\ell_2$ loss.

\begin{table}
	\caption{Test mAP (\%) of the student network on Pascal VOC2007. \textbf{Bold} denotes the best results.}
	\label{tab:VOC-class}
	\centering
	\begin{tabular}{l|cc}
		\toprule
		Teacher & ResNet101 & ResNet101\\
		Student & ResNet50  & MobileNetV2\\
		\midrule
		Teacher             & 93.27 & 93.27 \\
		Student 		    & 92.76 & 89.53 \\
		\midrule
		LSHL2 $\rightarrow$ KD     & $92.69\,(14\%\downarrow)$ & $89.64\,(3\%\uparrow)$ \\
		LSHL2 $\rightarrow$ L2    & $\textbf{93.17}\,(80\%\uparrow)$ & $\textbf{90.14}\,(16\%\uparrow)$ \\
		LSHL2 $\rightarrow$ LSH   & $92.40\,(71\%\downarrow)$ & $89.91\,(10\%\uparrow)$\\
		LSHL2 $\rightarrow$ LSHL2 & $92.85\,(18\%\uparrow)$ & $89.90\,(10\%\uparrow)$\\
		\bottomrule
	\end{tabular}
\end{table}

\begin{table}
	\caption{Test mAP (\%) of the student network on MS-COCO. \textbf{Bold} denotes the best results.}
	\label{tab:COCO-class}
	\centering
	\begin{tabular}{l|cc}
		\toprule
		Teacher  & ResNet101 & ResNet101\\
		Student  & ResNet50  & MobileNetV2\\
		\midrule
		Teacher             & 77.67 & 77.67 \\
		Student 		    & 75.54 & 71.06 \\
		\midrule
		LSHL2 $\rightarrow$ KD     & $75.14\,(19\%\downarrow)$ & $71.47\,(6\%\uparrow)$ \\
		LSHL2 $\rightarrow$ L2    	& $77.04\,(70\%\uparrow)$ & $73.28\,(34\%\uparrow)$ \\
		LSHL2 $\rightarrow$ LSH     & $76.59\,(49\%\uparrow)$ & $73.73\,(40\%\uparrow)$\\
		LSHL2 $\rightarrow$ LSHL2   & $\textbf{77.16}\,(76\%\uparrow)$ & $\textbf{73.73}\,(40\%\uparrow)$ \\
		\bottomrule
	\end{tabular}
\end{table}

We conduct experiments on VOC2007 and MS-COCO and adopt two settings, i.e., using ResNet101 to teach ResNet50 and MobileNetV2, respectively. The student is finetuned with the two-stage strategy, and the LSHL2 loss is used in the first stage based on the above findings. Table~\ref{tab:VOC-class} presents the results on VOC2007. LSHL2 $\rightarrow$ L2 denotes using the LSHL2 loss in the first stage and the $\ell_2$ loss in the second stage. The hyperparameters are set as $\beta=0.5$, $std_{hash}=std_t$, and $N=4D_t$ in all experiments. LSHL2 $\rightarrow$ L2 achieves the best performance. ResNet50 is improved by $0.41\%$ and MobileNetV2 is improved by $0.61\%$. Experimental results of MS-COCO are showed in Table~\ref{tab:COCO-class}. And we use $\beta=3$, $std_{hash}=std_t$, and $N=4D_t$ in all experiments. LSHL2 $\rightarrow$ LSHL2 achieves the best performance.

\begin{table}
	\caption{Test mAP (\%) on MS-COCO. The backbone networks use the global maximum pooling (GMP) to aggregate features. \textbf{Bold} denotes the best results.}
	\label{tab:COCO-class-soa}
	\centering
	\begin{tabular}{l|r@{.}lr@{.}lr@{.}l}
		\toprule
		Model				  & \multicolumn{2}{c}{MobileNetV2} & \multicolumn{2}{c}{ResNet50} & \multicolumn{2}{c}{ResNet101} \\
		\midrule
		Baseline             			& 73&90 & 77&20 & 79&57 \\
		\midrule
		MCAR~\cite{MCAR}    & 75&0 & \textbf{82}&\textbf{1} & \textbf{83}&\textbf{8} \\
		Ours       		& \textbf{76}&\textbf{03} & 79&55 & 81&24 \\
		\bottomrule
	\end{tabular}
\end{table}

A common trick in the multi-label classification task is replacing the global average pooling (GAP) with the global maximum pooling (GMP). So we evaluate the backbone network with GMP on the MS-COCO. Table~\ref{tab:COCO-class-soa} presents the results. As previously mentioned, we use ResNet101 (GMP) to teach MobileNetV2 (GMP) and ResNet50 (GMP). In addition, we also evaluate the performance of self-distillation, i.e., using ResNet101 (GMP) to teach ResNet101 (GMP). Our method achieves better performances than baselines. We compared our method with MCAR~\cite{MCAR}, which employs a complex training pipeline designed for multi-label classification and is the state-of-the-art method on multi-label classification. Our MobileNetV2 surprisingly surpassed that in MCAR, which demonstrates the advantage of our method.

\subsection{Detection}

\begin{figure*}
	\centering 
	\subfloat[]{\label{fig:fasterrcnn}\includegraphics[width=0.48\linewidth]{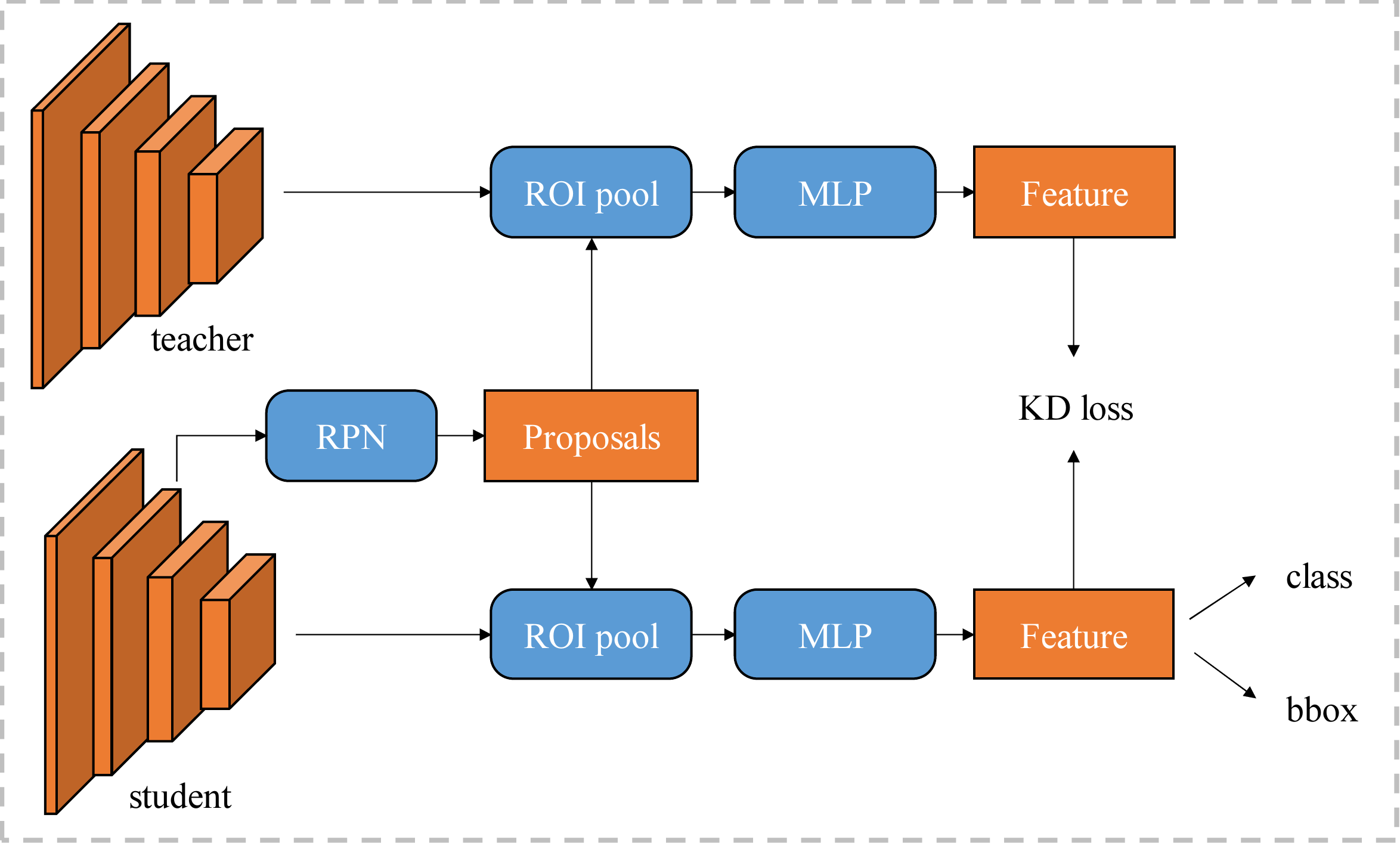}}\quad
	\subfloat[]{\label{fig:retinanet}\includegraphics[width=0.48\linewidth]{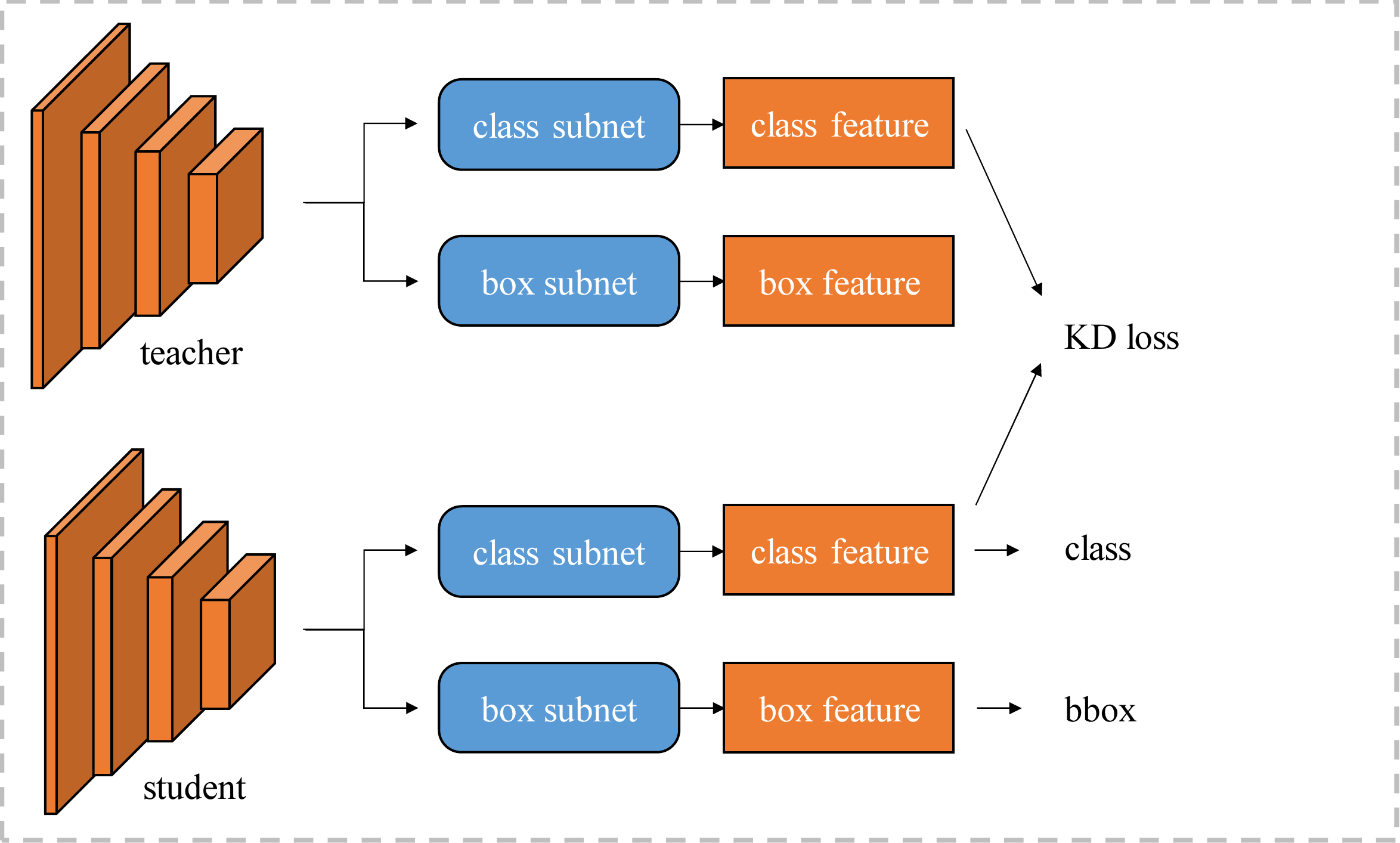}}\\
	\caption{The pipeline of our method on the object detection task. \protect\subref{fig:fasterrcnn} and \protect\subref{fig:retinanet} show our feature mimicking framework with Faster-RCNN and RetinaNet, respectively. This figure is best viewed in color and zoomed in.}
	\label{fig:detection} 
\end{figure*}

We evaluate our method on the object detection task. Following previous work~\cite{PAD}, we conduct experiments on the Pascal VOC dataset~\cite{VOC}. The training set consists of the VOC2007 trainval set and the VOC2012 trainval set, and in total 21K images. The testing set is the VOC2007 test set of 5K images. We use mAP@0.5 as the metric to compare the performance of different methods. The detection frameworks we adopted are both two-stage (Faster-RCNN~\cite{Faster-rcnn}) and one-stage (RetinaNet~\cite{RetinaNet}). And we use four networks (ResNet50, ResNet101, VGG11, VGG16) pretrained on ImageNet as the backbone. FPN~\cite{FPN} layers are adopted in all experiments. All models are finetuned on VOC with 24 epochs. The hyperparameters in feature mimicking loss are set as $\beta=7$, $std_{hash}=std_t$, $N=4D_t$ and $bias=0$ in all experiments. We have released our code.\footnote{\url{https://git.nju.edu.cn/wanggh/detection.vision}}

\begin{table}
	\caption{Test mAP@0.5 (\%) of the student network on Pascal VOC0712. The detector is Faster R-CNN with different backbones. \textbf{Bold} denotes the best results.}
	\label{tab:VOC-fasterrcnn}
	\centering
	\begin{tabular}{c|cc}
		\toprule
		Teacher & ResNet101 & VGG16 \\
		Student & ResNet50  & VGG11 \\
		\midrule
		Teacher              	& 83.6 & 79.0 \\
		Student 		    	& 82.0 & 75.1  \\
		\midrule
		ROI-mimic~\cite{ROI_mimic} 		 & $82.3\,(19\%\uparrow)$ & $75.0\,(3\%\downarrow)$ \\
		PAD-ROI-mimic~\cite{PAD} 		 		& $82.5\,(31\%\uparrow)$ & $75.8\,(18\%\uparrow)$ \\
		Fine-grained~\cite{DistillOD_FG}        & $82.0\,(0\%\uparrow)$ & $74.6\,(13\%\downarrow)$ \\
		PAD-Fine-grained~\cite{PAD}   			 & $82.3\,(19\%\uparrow)$ & $75.2\,(3\%\uparrow)$ \\
		\rowcolor{Shade}
		Ours (L2)                  & $83.0\,(63\%\uparrow)$ & $76.9\,(46\%\uparrow)$ \\
		\rowcolor{Shade}
		Ours (LSHL2)               	& $\textbf{83.1}\,(69\%\uparrow)$ & $\textbf{77.2}\,(54\%\uparrow)$\\
		\bottomrule
	\end{tabular}
\end{table}

Figure~\ref{fig:fasterrcnn} shows our feature mimicking framework with Faster-RCNN. As in classification, we want the student to mimic features in the penultimate layer. In the object detection framework, two linear layers are applied on the penultimate layer to generate the classification and bounding box predictions, respectively. Given one image, the backbone and FPN produce the feature pyramid, and the region proposal network (RPN) generates proposals to indicate the localities that objects may appear. Hence, many features are extracted according to the proposals. To make sure the student will mimic teacher's features in the same locations, the teacher uses the proposals produced by the student. When training this framework, we only add the proposed loss to the original loss and apply the traditional training strategy. Our proposed loss is applied on the entire detection network, and it affects the optimization of the backbone network, FPN, RPN and MLP.

The experimental results are presented in Table~\ref{tab:VOC-fasterrcnn}. First, we use ResNet101 to teach ResNet50. The performances of these two baseline networks are $83.6\%$ and $82.0\%$, respectively. The teacher is higher than student by $1.6\%$. All experimental results of ROI-mimic, PAD-ROI-mimic, Fine-grained and PAD-Fine-grained are cited from PAD~\cite{PAD}. They improve the student by at most $0.5\%$. With our feature mimicking framework, i.e., mimicking the features in the penultimate layer, simply using the $\ell_2$ loss as the feature mimicking loss can improve the student by $1\%$. That shows the benefit of feature mimicking on object detection. Combining the LSH loss and the $\ell_2$ loss, the student is improved by $1.1\%$. When using VGG16 to teach VGG11, the $\ell_2$ loss can improve the student by $1.8\%$. With the LSH loss, the student is improved by $2.1\%$.

\begin{table}
	\caption{Test mAP@0.5 (\%) of the student network on Pascal VOC0712. The detector is RetinaNet with different backbones. \textbf{Bold} denotes the best results.}
	\label{tab:VOC-retinanet}
	\centering
	\begin{tabular}{c|cc}
		\toprule
		Teacher & ResNet101 & VGG16 \\
		Student & ResNet50  & VGG11 \\
		\midrule
		Teacher   & 83.0 & 76.6 \\
		Student 	& 82.5 & 73.2  \\
		\midrule
		Fine-grained~\cite{DistillOD_FG} & $81.5\,(200\%\downarrow)$ & $72.0\,(35\%\downarrow)$ \\
		PAD-Fine-grained~\cite{PAD}	    & $81.9\,(120\%\downarrow)$ & $73.2\,(0\%\downarrow)$ \\
		\rowcolor{Shade}
		Ours (L2)                  & $82.6\,(20\%\uparrow)$ & $74.8\,(47\%\uparrow)$ \\
		\rowcolor{Shade}
		Ours (LSHL2)           & $\textbf{83.0}\,(\textbf{100}\%\uparrow)$ & $\textbf{75.2}\,(59\%\uparrow)$\\
		\bottomrule
	\end{tabular}
\end{table}

Figure~\ref{fig:retinanet} shows our feature mimicking framework with RetinaNet. Different from Faster-RCNN, RetinaNet produces features on all positions of the feature pyramid, and each position will consider several anchors. With the groundtruth bounding boxes, only a few of positions are considered as positive and sent to the classification loss. We force the student to mimic the features on these positive positions and ignore the features on negative positions. RetinaNet uses class subnet and box subnet to generate class feature and box feature, respectively. We find that it is better to only mimic the class feature and ignore the box feature. So our proposed feature mimicking loss affects the optimization of the backbone network, FPN and the class subnet. Table~\ref{tab:VOC-retinanet} shows the experimental results. Similar to Faster-RCNN, our feature mimicking framework can improve the student with a large margin. ResNet50 is improved by $0.5\%$ whose performance is comparable to the teacher performance. And VGG11 is also improved by $2\%$.

Overall, these object detection experimental results demonstrate the advantages of our method. The LSHL2 loss is consistently better than the $\ell_2$ loss in all experiments. Note that the difference between the teacher and the student is smaller when compared to the differences in recognition tasks. However, the high relative improvement numbers and the consistent improvements across different experiments both verifies our proposed method is effective. In this paper, we only focus on mimicking the final features and leave mimicking proposals as the future work. However, only using feature mimicking has already improved the student by a large margin, and the RetinaNet with ResNet50 backbone is even comparable to the teacher performance.

Compared with multi-label classification, we find it does not need the two stage training strategy on object detection. We guess it may be due to the MLP layer and the subnet in Faster-RCNN and RetinaNet, respectively. These layers are randomly initialized before finetuning on the detection dataset. The feature space alignment will be learned implicitly in these layers.

\section{Conclusion}
\label{sec:conclusion}

In this paper, we proposed a flexible and effective knowledge distillation method. We argued that mimicking feature in the penultimate layer is more advantageous than distilling the teacher's soft logits~\cite{KD}. And to make the student learn the more effective information from the teacher, it needs to give the student more freedom to its feature magnitude, but let it focus on mimicking the feature direction. We proposed a loss term based on Locality-Sensitive Hashing (LSH)~\cite{LSH} to fulfill this objective. Our algorithm was evaluated on single-label classification, multi-label classification and object detection. Experiments showed the effectiveness of the proposed method. 

Future work could explore how to improve our method, such as reducing the randomness in the LSH module, and
aligning feature spaces efficiently and even if without training data. Applying our method to other problems is also interesting. It is promising to combine our method with self-supervised learning. And we will also consider how to deploy our method to knowledge distillation under a data free setting. 

\bibliographystyle{IEEEtran}
\bibliography{egbib}

\begin{IEEEbiography}[{\includegraphics[width=1in,height=1.25in,clip,keepaspectratio]{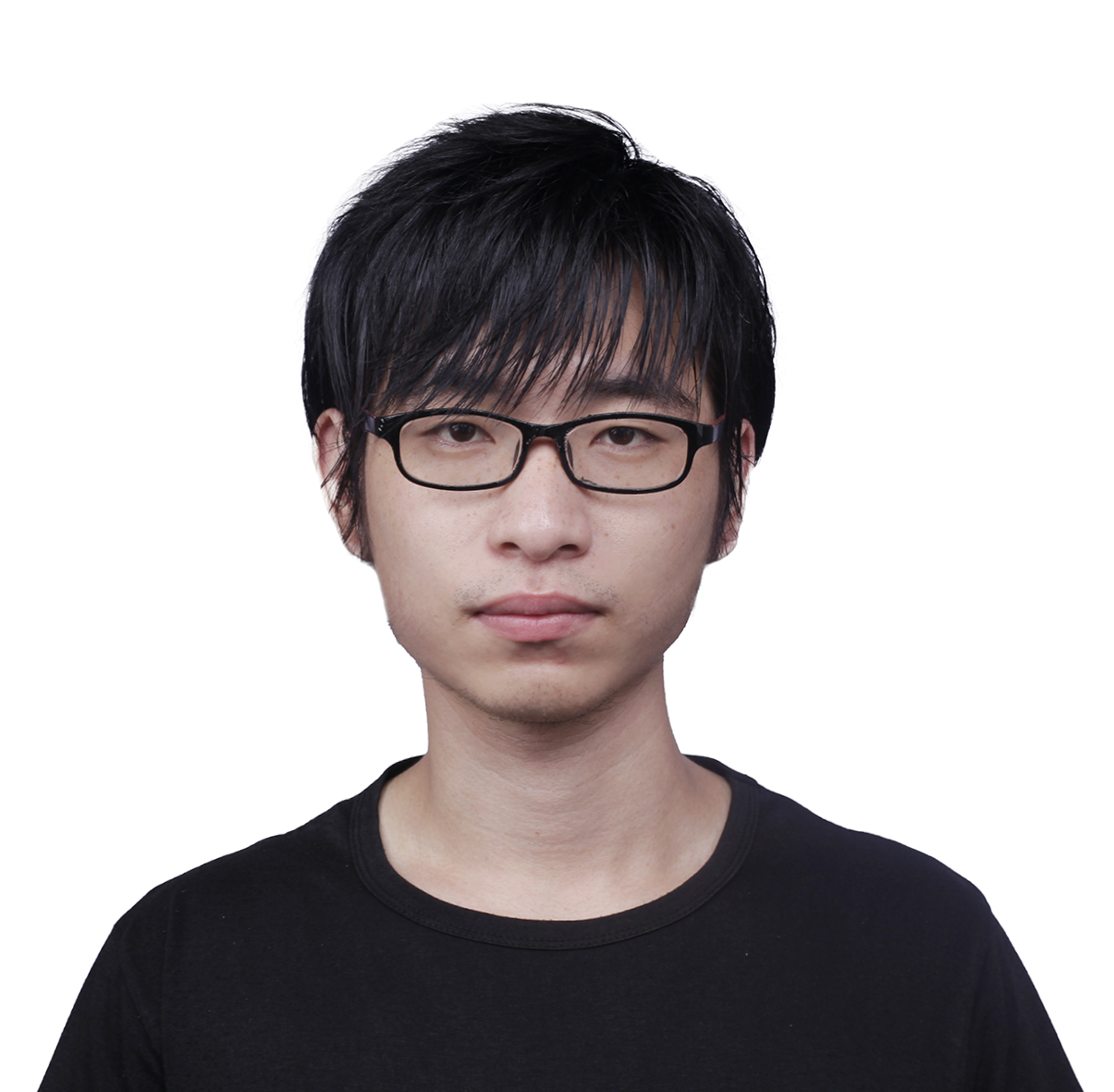}}]{Guo-Hua Wang}
	received  his  BS  degree  in the School of Management and Engineering from Nanjing University. He is currently a Ph.D. student in the Department of Computer Science and Technology in Nanjing University, China. His research interests are computer vision and machine learning.
\end{IEEEbiography}

\begin{IEEEbiography}[{\includegraphics[width=1in,height=1.25in,clip,keepaspectratio]{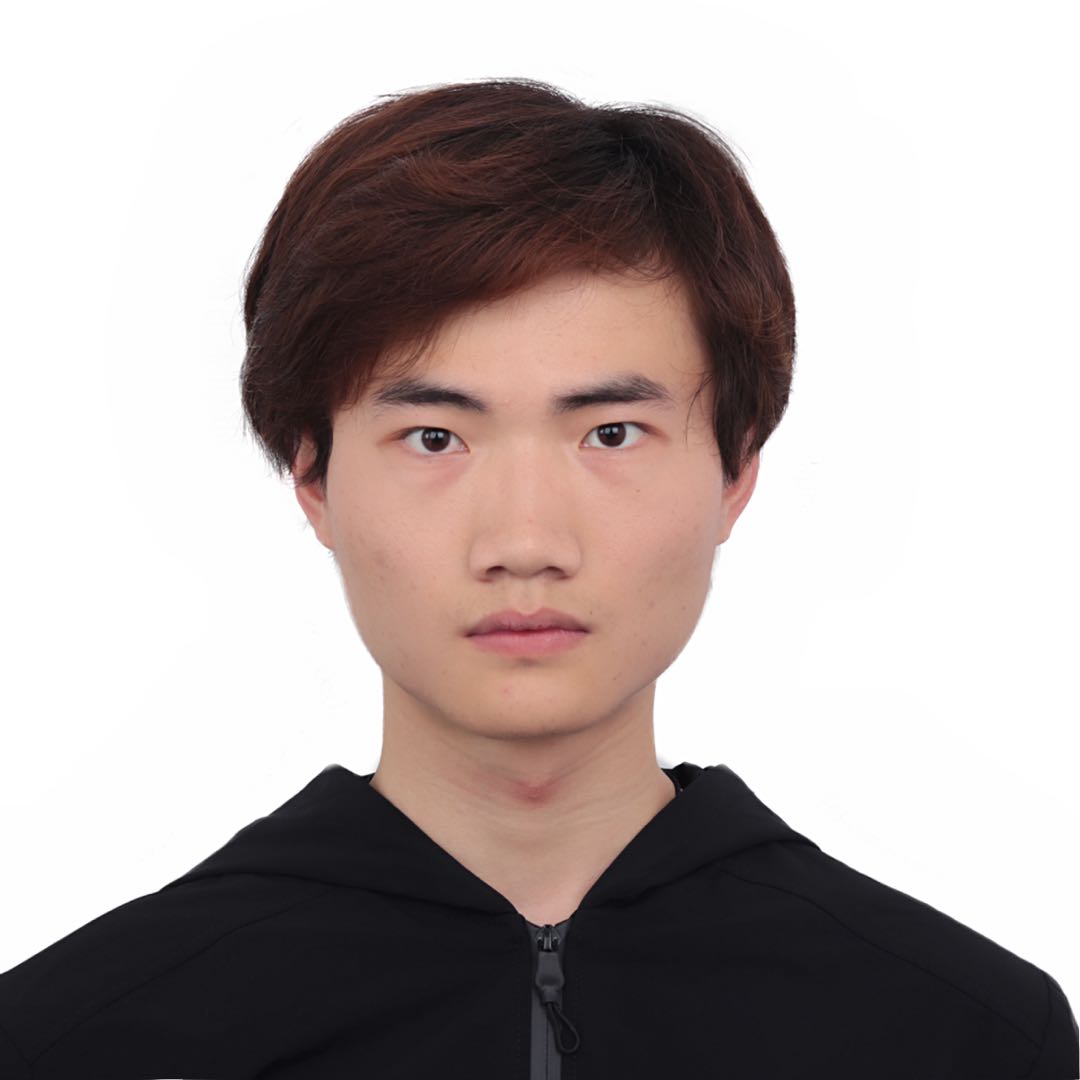}}]{Yifan Ge}
	received  his  BS  degree  in the Kuang Yaming Honors School from Nanjing University in 2019. He is currently a graduate student in the School of Artificial Intelligence at Nanjing University, China. His research interests include computer vision and machine learning.
\end{IEEEbiography}

\begin{IEEEbiography}[{\includegraphics[width=1in,height=1.25in,clip,keepaspectratio]{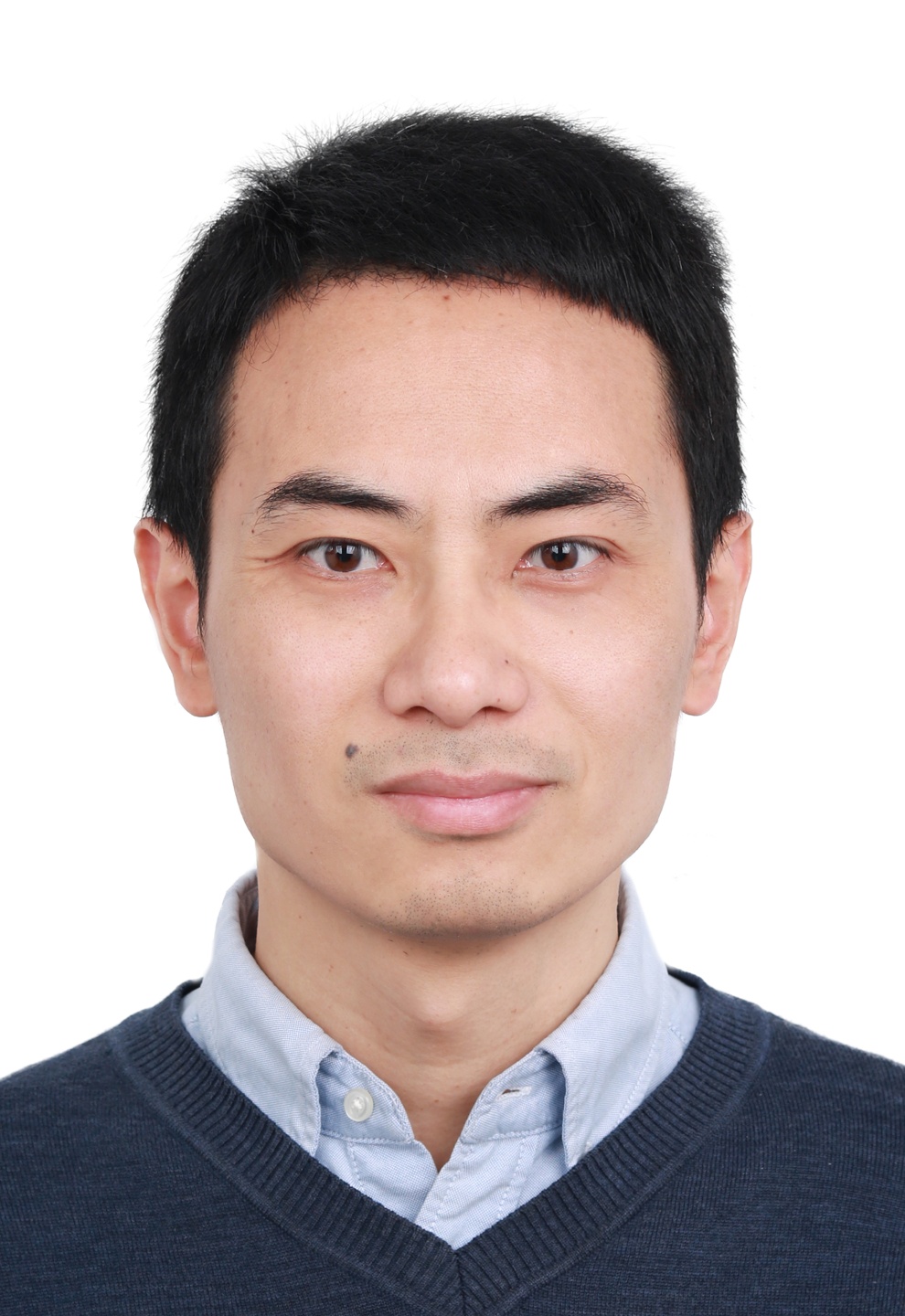}}]{Jianxin Wu}
	received his BS and MS degrees from Nanjing University, and his PhD degree from the Georgia Institute of Technology, all in computer science. He is currently a professor in the Department of Computer Science and Technology and the School of Artificial Intelligence at Nanjing University, China, and is associated with the State Key Laboratory for Novel Software Technology, China. He has served as an (senior) area chair for CVPR, ICCV, ECCV, AAAI and IJCAI, and as an associate editor for the IEEE Transactions on Pattern Analysis and Machine Intelligence. His research interests are computer vision and machine learning.
\end{IEEEbiography}

\clearpage

\appendices

\section{Proof of Claim~\ref{claim:1}}

\begin{proof}
	We will prove that $h_{\mW, \mathbf{b'}}(s\vf_t) =h_{\mW, \vb}(\vf_t)$, where $\mW=[\vw_1,\vw_2,\cdots,\vw_N]\tran$, $\mathbf{b'}=[b'_1,b'_2,\cdots,b'_N]\tran$ and $\vb=[b_1,b_2,\cdots,b_N]\tran$. Assume there are $m$ teacher features, that is $\vf_1,\vf_2,\cdots,\vf_m$.
	
	For $0\le j \le N$, $b_j=median(\vf_1,\vf_2,\cdots,\vf_m)$. And it is easy to see $b'_j=median(s\vf_1,s\vf_2,\cdots,s\vf_m)=sb_j$ when $s>0$. Hence, $\sign(\vw_j\tran\vf_i+b_j)=\sign(\vw_j\tran s\vf_i+sb_j)$ for $0\le i \le m$ when $s>0$. That implies $h_{\mW, \mathbf{b'}}(s\vf_t) =h_{\mW, \vb}(\vf_t)$.
\end{proof}

\section{Proof of Claim~\ref{claim:2}}

\begin{proof}
	Note that 
	\begin{equation}
		\Loss_{lsh}(\vf_t, \vf_s) = -\frac{1}{N}\sum_{j=1}^{N}\left[h_j\log p_j + (1-h_j)\log(1-p_j)\right]\,,
	\end{equation}
	where $h_j$ and $p_j$ is the $j$-th entry of $\sign(\mW\tran\vf_t)$ and $\sigma(\mW\tran\vf_s)$, respectively.

	First, we discuss the situation when $h_j=1$, that is, the angle between $\vw_j$ and $\vf_t$ is less than 90 degrees, and $\cos\langle\vw_j, \vf_t\rangle\geq 0$. Therefore, 
	\begin{align}
		\label{claim2:proof1}
		& -h_j\log\sigma(s\vw_j\tran\vf_s)-(1-h_j)\log(1-\sigma(s\vw_j\tran\vf_s)) \notag \\
		=&
		-\log \sigma(s\|\vw_j\|\|\vf_s\|\cos {\langle\vw_j, \vf_s\rangle}) \\
		\leq & -\log \sigma(\|\vw_j\|\|\vf_s\|\cos {\langle\vw_j, \vf_t\rangle}) \\
		=&-h_j\log p_j-(1-h_j)\log(1-p_j) \,.
	\end{align}
	
	Then, when $h_j=0$, similar to equation~\ref{claim2:proof1}, we can get
	\begin{align}
		& -h_j\log\sigma(s\vw_j\tran\vf_s)-(1-h_j)\log(1-\sigma(s\vw_j\tran\vf_s)) \notag \\
		=&
		-\log \left(1-\sigma(s\vw_j\tran\vf_s)\right) \\
		\leq & -\log \left(1-\sigma(\|\vw_j\|\|\vf_s\|\cos {\langle\vw_j, \vf_t\rangle})\right) \\
		= & -h_j\log p_j-(1-h_j)\log(1-p_j) \,.
	\end{align}
	
	To sum up, $\Loss_{lsh}(\vf_t, s\vf_s)\leq \Loss_{lsh}(\vf_t, \vf_s)$ always holds when $s>1$. 
\end{proof}

\section{Proof of Claim~\ref{claim:lsh_p} and Claim~\ref{claim:degree_p}}

We define the notations and terminologies first. We assume that $\vf_s$ and $\vf_t$ follow the standard normal distribution:
\begin{itemize}
	\item $\vf_t \in \R^D: {\vf_t} \sim \mathcal{N}(\mathbf{0}, \mathbf{I}_D)$
	\item $\vf_s \in \R^D: {\vf_t} \sim \mathcal{N}(\mathbf{0}, \mathbf{I}_D)$
\end{itemize}

In our LSH module, $\mW\in \R^{D \times N}$ can be alternatively written as $[\vw_1,\vw_2,\cdots,\vw_N]\tran$ and entries of $\mW$ are sampled from a Guassian distribution: 
\begin{itemize}
	\item $\vw_j \in \R^{D}: \vw_j \sim \mathcal{N}(\mathbf{0}, \mathbf{I}_{D})$
\end{itemize}

A few derived variables are: 
\begin{itemize}
	\item $\vh \in \R^{N}: h_j \doteq \sign{\left(\vw_j\tran\vf_t\right)}$
	\item $\vp \in \R^{N}: p_j \doteq \sigmoid{\left(\vw_j\tran\vf_s\right)}$
	\item $\vl \in \R^{N}: l_j \doteq -h_j\log{\left(p_j\right)} - \left(1 - h_j\right)\log{\left(1 - p_j\right)}$
\end{itemize}

We also use a few shorthand notations: 
\begin{itemize}
	\item $\dir{(\vx)} \doteq \begin{cases}
		\frac{\vx}{\normtwo{\vx}} & \text{if } \normtwo{\vx} > 0\,, \\
		\vzero & \text{otherwise}\,.
	\end{cases}$
	\item $\anglefn{(\vx, \vy)} \doteq \arccos{(\dir{(\vx)}\tran\dir{(\vy)})}$
	\item $\sphere{n} \doteq \left\{ \vx \in \R^n \mid \normtwo{\vx} = 1\right\}$: the $n$-dimensional unit hypersphere
	\item $d_{geo}$: the geodesic distance, with which $\sphere{n}$ forms a legitimate metric space
	\item $\mu_{n}$: the Lebesgue measure on $\R^n$
	\item $\sigma_{n-1}$: the surface area measure on $\sphere{n}$
	\item $A_{n-1} \doteq \int_{\sphere{n}}{{}\,\mathrm{d}\sigma_{n-1}} = \frac{2 \pi^{n / 2}}{\Gamma{(n / 2)}}$: the surface area of $\sphere{n}$
	\item $\mathbb{I}$: the indicator function
	\item $p(x)$: the p.d.f. of $x$
\end{itemize}

\begin{lemma}
\label{lem:unisp}
Let $\vx \in \R^n$ ($n \in \N^+$) be a random vector with each element $x_i \sim \mathcal{N}(0, 1)$ independently. Then for any function $f: \sphere{n} \cup \{\vzero\} \to \R$ satisfying
\begin{itemize}
	\item $f$ is bounded, 
	\item $f$ is continuous on $\sphere{n}$, 
\end{itemize}
there holds
\begin{align}
	\expect{f(\dir{(\vx)})}{\vx}
	= \frac{1}{A_{n-1}} \int_{\sphere{n}}{{f(\vu)} \,\mathrm{d}\sigma_{n-1}{(\vu)}}\,.
\end{align}
\end{lemma}

\begin{proof}
\begin{align}
	&\expect{f(\dir{(\vx)})}{\vx} \notag\\
	=& \int_{\R^n}{{f(\dir{(\vx)}) \cdot p(\vx)}\,\mathrm{d}\mu_{n}(\vx)} + \notag\\
	\quad &
	\int_{\{\vzero\}}{{f(\dir{(\vx)}) \cdot p(\vx)}\,\mathrm{d}\mu_{n}(\vx)} \\
	=& \int_{\R^n}{{f(\dir{(\vx)}) \cdot p(\vx)}\,\mathrm{d}\mu_{n}(\vx)} + \underbrace{0}_{\text{due to $f$'s boundedness}} \\
	=& \int_{\R^n}{{f(\dir{(\vx)}) \cdot {(2\pi)}^{-n/2} e^{ -{\frac{1}{2}} {\|\vx\|_2^2} }} \,\mathrm{d}\mu_{n}(\vx)} \\
	=& {(2\pi)}^{-n/2} \underbrace{\int_{0}^{\infty}{{{\left[ \int_{\sphere{n}}{{f(\vu)} \,\mathrm{d}\sigma_{n-1}{(\vu)}} \right]} \cdot e^{-\frac{r^2}{2}} r^{n-1}} \,\mathrm{d}r}}_{\text{integration by substitution (from Cartesian to polar)}}  \\
	=& (2\pi)^{-n/2} \int_{0}^{\infty}{\left(e^{-r^2/2}r^{n-1}\right)\,\mathrm{d}r} \int_{\sphere{n}}{{f(\vu)} \,\mathrm{d}\sigma_{n-1}{(\vu)}}  \\
	=& \underbrace{\frac{\Gamma{(n / 2)}}{2 \pi^{n / 2}}}_{\frac{1}{A_{n-1}}} \int_{\sphere{n}}{{f(\vu)} \,\mathrm{d}\sigma_{n-1}{(\vu)}} \,.
\end{align}
\end{proof}

\begin{corollary}
\label{cor:indicator}
Let $\vx \in \R^n$ ($n \in \N^+$) be a random vector with each element $x_i \sim \mathcal{N}(0, 1)$ independently. Then for any Borel set $\sB$ in $\sphere{n}$, 
\begin{align}
	\prob{\dir{(\vx)} \in \sB}
	= \frac{\sigma_{n-1}{(\sB)}}{A_{n-1}} \,.
\end{align}
\end{corollary}

\begin{proof}
For an open set $\sO$ in $\sphere{n}$ (rename it to make things clear), 
define 
\begin{align}
	f_{\sO}{(\vu)} = 
	\begin{cases}
		0, & \text{for } \vu = \vzero; \\
		\chi_{\sO}, & \text{for } \vu \in \sphere{n} \,,
	\end{cases}
\end{align}
where $\chi_{\sO}$ is the characteristic function of $\sO$ and 
\begin{align}
	f_{\sO}^{(k)}{(\vu)} = 
	\begin{cases}
		0 & \text{if } \vu = \vzero; \\
		\max{(0, 1 - k \cdot \inf_{\vv \in \sO}{d_{geo}{(\vu, \vv)})}} &\text{if } \vu \in \sphere{n}
	\end{cases}
\end{align}
for $k \in \sN^*$. Then the conditions of Lebesgue's dominated convergence theorem are met, i.e., 
\begin{itemize}
	\item $f_{\sO}^{(k)}$'s are bounded, 
	\item $f_{\sO}^{(k)}$'s converge pointwise to $f_{\sO}$. 
\end{itemize}
Thus 
\begin{align}
	&\prob{\dir{(\vx)} \in \sO} \notag\\
	=& \, \expect{f_{\sO}{(\dir{(\vx)})}}{\vx} \\
	=& \underbrace{\lim_{k \to \infty}{\expect{f_{\sO}^{(k)}{(\dir{(\vx)})}}{\vx}}}_{\text{due to dominated convergence theorem}} \\
	=& \lim_{k \to \infty}{\underbrace{\left(\frac{1}{A_{n-1}} \int_{\sphere{n}}{{f_{\sO}^{(k)}(\vu)} \,\mathrm{d}\sigma_{n-1}{(\vu)}}\right)}_{\text{due to \Lemref{lem:unisp}}}} \\
	=& \frac{1}{A_{n-1}} \lim_{k \to \infty}{\left(\int_{\sphere{n}}{{f_{\sO}^{(k)}(\vu)} \,\mathrm{d}\sigma_{n-1}{(\vu)}}\right)} \\
	=& \frac{1}{A_{n-1}} \underbrace{\int_{\sphere{n}}{{f_{\sO}(\vu)} \,\mathrm{d}\sigma_{n-1}{(\vu)}}}_{\text{due to dominated convergence theorem}} \\
	=& \frac{1}{A_{n-1}} \int_{\sphere{n}}{{\chi_{\sO}(\vu)} \,\mathrm{d}\sigma_{n-1}{(\vu)}} \\
	=& \frac{1}{A_{n-1}} \int_{\sO}{ \,\mathrm{d}\sigma_{n-1}{(\vu)}} \\
	=& \frac{\sigma_{n-1}{(\sO)}}{A_{n-1}}\,.
\end{align}

Now that the equation holds for any open set $\sO$, it can be shown by induction that 
\begin{align}
	\prob{\dir{(\vx)} \in \sB}
	= \frac{\sigma_{n-1}{(\sB)}}{A_{n-1}}
\end{align}
for any Borel set $\sB$. 
\end{proof}

\begin{lemma}
\label{lem:angpdf}
\begin{equation}
	\pd{\anglefn{{(\vf_t, \vf_s)}} = \theta} = \frac{A_{D-2}}{A_{D-1}} \sin^{D-2}{(\theta)}
\end{equation}
for $\theta \in (0, \pi)$.
\end{lemma}

\begin{proof}
	\begin{align}
		& \prob{\anglefn{{(\vf_t, \vf_s)}} \le \theta \mid \vf_t} \notag\\
		=& \prob{\dir{(\vf_s)} \in \{ \vx \in \sphere{D} \mid d_{geo}{(\dir{(\vf_t)}, \vx)} \le \theta \} \mid \vf_t} \\
		=& \underbrace{ \frac{1}{A_{D-1}} {\sigma_{D-1}{\left( \left\{ \vx \in \sphere{D} \mid d_{geo}{(\dir{(\vf_t)}, \vx)} \le \theta \right\} \right)}} }_{\text{due to \Corref{cor:indicator}}} \\
		=& \frac{1}{A_{D-1}} \int_{0}^{\theta}{{\left(A_{D-2} \sin^{D-2}{(\phi)}\right)}\,\mathrm{d}\phi} \\
		=& \frac{A_{D-2}}{A_{D-1}} \int_{0}^{\theta}{{\sin^{D-2}{(\phi)}}\,\mathrm{d}\phi}\,.
	\end{align}
	Then 
	\begin{align}
		&\prob{\anglefn{{(\vf_t, \vf_s)}} \le \theta} \notag\\
		=& \int_{\R^D}{\left( \prob{\anglefn{{(\vf_t, \vf_s)}} \le \theta \mid \vf_t} \cdot p(\vf_t)\right)\, \mathrm{d} \mu_{D}{(\vf_t)}} \\
		=& \int_{\R^D}{\left( \left( \frac{A_{D-2}}{A_{D-1}} \int_{0}^{\theta}{{\sin^{D-2}{(\phi)}}\,\mathrm{d}\phi}\right) \cdot p(\vf_t)\right)\, \mathrm{d} \mu_{D}{(\vf_t)}} \\
		=& \left( \int_{\R^D}{\left( p(\vf_t)\right)\, \mathrm{d} \mu_{D}{(\vf_t)}}\right) \left(\frac{A_{D-2}}{A_{D-1}} \int_{0}^{\theta}{{\sin^{D-2}{(\phi)}}\,\mathrm{d}\phi}\right) \\
		=& \frac{A_{D-2}}{A_{D-1}} \int_{0}^{\theta}{{\sin^{D-2}{(\phi)}}\,\mathrm{d}\phi}\,.
	\end{align}
	Thus
	\begin{align}
		\pd{\anglefn{{(\vf_t, \vf_s)}} = \theta}
		=& \deri{}{\theta} \prob{\anglefn{{(\vf_t, \vf_s)}} \le \theta} \\
		=& \deri{}{\theta} \left( \frac{A_{D-2}}{A_{D-1}} \int_{0}^{\theta}{{\sin^{D-2}{(\phi)}}\,\mathrm{d}\phi} \right) \\
		=& \frac{A_{D-2}}{A_{D-1}}\deri{}{\theta} \int_{0}^{\theta}{{\sin^{D-2}{(\phi)}}\,\mathrm{d}\phi} \\
		=& \frac{A_{D-2}}{A_{D-1}} \sin^{D-2}{(\theta)} \,.
	\end{align}
\end{proof}

\subsection{Proof of Claim~\ref{claim:lsh_p}}

\begin{proof}
	First, let us inspect the properties of the $l_j$'s, which can be 
	rewritten as 
	\begin{align}
		l_j = 
		\begin{cases}
			-\log{\left(p_j\right)} & \text{if } h_j = 1\,,\\
			-\log{\left(1 - p_j\right)} & \text{if } h_j = 0\,. 
		\end{cases}
	\end{align}
	Note that $\log{2} = -\log{\left(1 - \frac{1}{2}\right)}$. 
	Thus $l_j < \log{2}$ if and only if 
	\begin{align}
		\vp[j]
		\begin{cases}
			> \frac{1}{2} & \text{if } \vh[j] = 1\,,\\
			< \frac{1}{2} & \text{if } \vh[j] = 0\,,
		\end{cases}
	\end{align}
	which is equivalent to 
	\begin{align}
		\vw_j\tran\vf_s
		\begin{cases}
			> 0 & \text{if } \vw_j\tran\vf_t > 0\,,\\
			< 0 & \text{if } \vw_j\tran\vf_t \le 0\,. 
		\end{cases}
	\end{align}
	In other words, 
	\begin{align}
		\dir{(\vw_j)} \in \sL_{\vf_t, \vf_s}\,,
	\end{align}
	where 
	\begin{align}
		\sL_{\vf_t, \vf_s} = &
		\left\{
		\vx \in \sphere{D} \mid 
		\left(\left(\vx \tran \vf_s > 0\right) \wedge \left(\vx \tran \vf_t > 0\right)\right) 
		\right\}
		\vee \notag\\
		&\left\{
		\vx \in \sphere{D} \mid 
		\left(\left(\vx \tran \vf_s < 0\right) \wedge \left(\vx \tran \vf_t \le 0\right)\right) 
		\right\}
	\end{align}
	is the union of two lunes. 

	Applying \Corref{cor:indicator}, we have
	\begin{align}
		& \prob{\dir{(\vw_j)} \in \sL_{\vf_t, \vf_s} \mid \vf_t, \vf_s} \notag\\
		=& \frac{\sigma_{D-1}{(\sL_{\vf_t, \vf_s})}}{A_{D-1}} \\
		=& \frac{2}{A_{D-1}}\underbrace{\int_{0}^{\frac{\pi }{2}}\frac{\pi - \anglefn{(\vf_t, \vf_s)}}{2\pi} A_1 \cos{(\theta)} A_{D-3} \sin^{D-3}{(\theta)} \mathrm{d}\theta}_{\text{$\sphere{D}$ viewed as a union of tori}} \\
		=& \frac{\pi - \anglefn{(\vf_t, \vf_s)}}{\pi A_{D-1}} \int_{0}^{\frac{\pi }{2}} \left( A_1 \cos{(\theta)} \cdot A_{D-3} \sin^{D-3}{(\theta)} \right) \,\mathrm{d}\theta \\
		=&\frac{\pi - \anglefn{(\vf_t, \vf_s)}}{\pi A_{D-1}} A_{D-1}\\
		=& 1 - \frac{\anglefn{(\vf_t, \vf_s)}}{\pi}\,.
	\end{align}
	Thus 
	\begin{align}
		& \prob{\dir{(\vw_j)} \in \sL_{\vf_t, \vf_s} \mid \anglefn{(\vf_t, \vf_s)} = \theta} \notag\\
		=& \int_{\R^{D \times D}}(\prob{\dir{(\vw_j)} \in \sL_{\vf_t, \vf_s} \mid \vf_t, \vf_s} \cdot \notag\\
		&\qquad\pdfunc{\vf_t, \vf_s \mid \anglefn{(\vf_t, \vf_s)} = \theta})\mathrm{d} \mu_{D \times D}{(\vf_t, \vf_s)} \\
		=& \int_{\R^{D \times D}}\left( 1 - \frac{\theta}{\pi} \right) \cdot \pdfunc{\vf_t, \vf_s \mid \anglefn{(\vf_t, \vf_s)} = \theta}\mathrm{d} \mu_{D \times D}{(\vf_t, \vf_s)} \\
		=& \left( 1 - \frac{\theta}{\pi} \right) \int_{\R^{D \times D}} \pdfunc{\vf_t, \vf_s \mid \anglefn{(\vf_t, \vf_s)} = \theta} \,\mathrm{d} \mu_{D \times D}{(\vf_t, \vf_s)} \\
		=& 1 - \frac{\theta}{\pi}\,.
	\end{align}
\end{proof}

\subsection{Proof of Claim~\ref{claim:degree_p}}

\begin{proof}
	\begin{align}
		\label{eq:c3e1}
		& \prob{\bigwedge_{j = 1}^{N}{\left(l_j < \log{2}\right)} \mid \anglefn{(\vf_t, \vf_s)} = \theta} \nonumber \\
		=& \underbrace{\prod_{j=1}^{N}{\prob{l_j < \log{2} \mid \anglefn{{(\vf_t, \vf_s)}} = \theta}}}_{\text{due to conditional independence of $l_j$'s}} \\
		=& \underbrace{\left(1 - \frac{\theta}{\pi}\right)^{N}}_{\text{due to Claim~\ref{claim:lsh_p}}} \,.
	\end{align}
	Applying the Bayes' rule, the conditional probability density of $\anglefn{(\vf_t, \vf_s)}$ can be derived as 
	\begin{align}
		\label{eq:c3e2}
		& \pd{ \anglefn{(\vf_t, \vf_s)} = \theta \mid \bigwedge_{j = 1}^{N}{\left(l_j < \log{2}\right)}} \nonumber \\
		=& \frac{\pd{\bigwedge_{j = 1}^{N}{\left(l_j < \log{2}\right)} \mid \anglefn{(\vf_t, \vf_s)} = \theta} \cdot \pd{ \anglefn{(\vf_t, \vf_s)} = \theta }}{\pd{\bigwedge_{j = 1}^{N}{\left(l_j < \log{2}\right)}}}  \\
		=& \frac{\pd{\bigwedge_{j = 1}^{N}{\left(l_j < \log{2}\right)}|\anglefn{(\vf_t, \vf_s)} = \theta} \pd{ \anglefn{(\vf_t, \vf_s)} = \theta }}{\int_{0}^{\pi}{{\pd{\bigwedge_{j = 1}^{N}{\left(l_j < \log{2}\right)}|\anglefn{(\vf_t, \vf_s)} = \theta} \pd{ \anglefn{(\vf_t, \vf_s)} = \theta }}\mathrm{d}\theta}}  \\
		=& \underbrace{\frac{\left(1 - \frac{\theta}{\pi}\right)^{N} \cdot \frac{A_{D-2}}{A_{D-1}} \sin^{D-2}{(\theta)}}{\int_{0}^{\pi}{{\left(\left(1 - \frac{\theta}{\pi}\right)^{N} \cdot \frac{A_{D-2}}{A_{D-1}} \sin^{D-2}{(\theta)} \right)} \,\mathrm{d}\theta}}}_\text{due to \Eqref{eq:c3e1} and \Corref{lem:angpdf}} \\
		=& \frac{\left(1 - \frac{\theta}{\pi}\right)^{N} \cdot \sin^{D-2}{(\theta)}}{\int_{0}^{\pi}{{\left(\left(1 - \frac{\theta}{\pi}\right)^{N} \cdot \sin^{D-2}{(\theta)} \right)} \,\mathrm{d}\theta}} \,.
	\end{align}
	Thus 
	\begin{align}
		& \prob{\anglefn{{(\vf_t, \vf_s)}} < \epsilon \mid \bigwedge_{j = 1}^{N}{\left(l_j < \log{2}\right)}} \notag \\
		=& \int_{0}^{\epsilon}{{\pd{ \anglefn{(\vf_t, \vf_s)} = \theta \mid \bigwedge_{j = 1}^{N}{\left(l_j < \log{2}\right)} }} \,\mathrm{d}\theta} \\
		=& \int_{0}^{\epsilon}{\underbrace{\left( \frac{\left(1 - \frac{\theta}{\pi}\right)^{N} \cdot \sin^{D-2}{(\theta)}}{\int_{0}^{\pi}{{\left(\left(1 - \frac{\theta}{\pi}\right)^{N} \cdot \sin^{D-2}{(\theta)} \right)} \,\mathrm{d}\theta}} \right)}_{\text{due to \Eqref{eq:c3e2}}} \,\mathrm{d}\theta} \\
		=& \frac{\int_{0}^{\epsilon}{{\left(\left(1 - \frac{\theta}{\pi}\right)^{N} \cdot \sin^{D-2}{(\theta)} \right)} \,\mathrm{d}\theta}}{\int_{0}^{\pi}{{\left(\left(1 - \frac{\theta}{\pi}\right)^{N} \cdot \sin^{D-2}{(\theta)} \right)} \,\mathrm{d}\theta}} \,.
	\end{align}
\end{proof}

\end{document}